\documentclass[11pt]{article}

\usepackage{amsmath}
\usepackage{amssymb}
\usepackage{latexsym}
\usepackage[numbers]{natbib}
\usepackage{makeidx}
\usepackage{nomencl}
\usepackage{graphicx}
\usepackage[arrow, matrix, curve]{xy}
\usepackage{mathrsfs}
\usepackage{dsfont}   
\usepackage{hyperref}

 \newtheorem{theorem}{Theorem}[section]
 
 \newtheorem{lemma}[theorem]{Lemma}
 \newtheorem{remark}[theorem]{Remark}

 \newtheorem{proposition}[theorem]{Proposition}
 \newtheorem{corollary}[theorem]{Corollary}
 \newtheorem{assumption}[theorem]{Assumption}

 \newenvironment{proof}
   {\begin{list}{\textbf{Proof}:}
                {\setlength{\leftmargin}{0em}
                 \setlength{\labelwidth}{-0.5em}
                }
   }
   {\hspace*{\fill}$\Box$\end{list}}
   
 \newenvironment{proof*}
   {\begin{list}{\textbf{Proof}:}
                {\setlength{\leftmargin}{0em}
                 \setlength{\labelwidth}{-0.5em}
                }
   }
   {\end{list}}

  \newcommand{\R}{\mathds{R}}
  \newcommand{\N}{\mathds{N}}

\newcommand{\la}{\langle}
\newcommand{\ra}{\rangle_{H}}

\newcommand{\SVMplz}{f_{P,\lambda_{0}}}

\newcommand{\Ls}{L^{\prime}}
\newcommand{\Lss}{L^{\prime\prime}}

\newcommand{\X}{\mathcal{X}}
\newcommand{\Y}{\mathcal{Y}}

\newcommand{\XX}{\X\times\X}
\newcommand{\XY}{\X\times\Y}

\newcommand{\G}{\mathcal{G}}

\newcommand{\Transp}{^\textsf{\tiny T}}

\begin{document}

\title{Asymptotic Confidence Sets for   
       General Nonparametric Regression and Classification
       by Regularized Kernel Methods}

 \author{Robert Hable$^{\,\ast}$ \\ 
         Department of Mathematics \\
         University of Bayreuth \\
         D-95440 Bayreuth, Germany
         }
 \date{}
         
\maketitle

\begin{abstract}
  Regularized kernel methods such as, e.g.,
  support vector machines and 
  least-squares support vector regression
  constitute an important class of standard learning
  algorithms in machine learning. 
  Theoretical investigations 
  concerning asymptotic properties
  have manly focused on rates of convergence
  during the last years but there are only very few 
  and limited (asymptotic) results on statistical
  inference so far. As this is a serious limitation
  for their use in mathematical statistics,
  the goal of the article is to fill this gap.
  Based on asymptotic normality of many of these 
  methods, the article derives a strongly consistent
  estimator for the unknown covariance matrix of the 
  limiting normal distribution.
  In this way, we obtain asymptotically correct
  confidence sets for $\psi(f_{P,\lambda_0})$
  where $f_{P,\lambda_0}$ denotes the minimizer
  of the regularized risk in the reproducing kernel
  Hilbert space $H$ and $\psi:H\rightarrow\mathds{R}^m$
  is any Hadamard-differentiable functional.
  Applications include (multivariate) pointwise
  confidence sets for values of $f_{P,\lambda_0}$
  and confidence sets for gradients, integrals,
  and norms.
\end{abstract}
  

\section{Introduction}  \label{sec-introduction}

Regularized kernel methods constitute an important class of
standard learning algorithms in machine learning theory.
The prominent learning algorithms \emph{support vector machine}
(for classification) and 
\emph{(least-squares) support vector regression} 
(for regression) also belong to this class; see, e.g.,
\cite{vapnik1998}, \cite{schoelkopf2002}, and \cite{steinwart2008}.
While these methods are standard in machine learning theory
and are widely applied, their propagation
in mathematical statistics
is still limited. This is partly due to the fact that
there is a lack of results on statistical inference for these methods
so far.
In machine learning theory, the goal in a supervised learning problem
is to find a ``good''
predictor $f:\X\rightarrow\Y$ which maps
the observed value $x\in\X$ of an input variable $X$ to a prediction of
the unobserved value $y\in\Y$ of an output variable.
A learning algorithm $S_n$ is a mapping 
which maps a set $D_n$ of observed training data 
$(x_1,y_1),\dots,(x_n,y_n)$ to a predictor $f_{D_n}$.
In mathematical statistics, such a problem would rather be called
a (nonparametric) regression (or a classification) problem, 
$X$ is the covariate,
$Y$ is the response variable, $S_n$ is an estimator and
$S_n(D_n)=f_{D_n}$ is the estimated function.
In both contexts, a ``good'' predictor/estimate $f$
has a small expected loss (also called risk) 
$\mathcal{R}_{L,P}(f)=\int L(x,y,f(x))\,P(d(x,y))$ where
$L$ is a ``suitable'' loss function and $P$ is the
joint distribution of $X$ and $Y$.  
However, depart form the different terminology, there is also a real
difference: 
In machine learning, the goal is to find any
predictor $f$ which has a small risk and, accordingly, 
a learning algorithm $S_n$ should be risk-consistent, i.e.,
$\mathcal{R}_{L,P}(S_n(D_n))$ converges to the infimal risk
$\inf_f\mathcal{R}_{L,P}(f)$ for $n\rightarrow\infty$.
In statistics, it is e.g.\ common to make a signal
plus noise assumption such as $Y=f_0(X)+g(X)\varepsilon$ and
the goal is to estimate the unknown regression function
$f_0$. Under suitable assumptions, $f_0$ minimizes
$\mathcal{R}_{L,P}(f)$ in certain sets of functions $f$.
While, in machine learning, one is mainly interested in
minimizing the risk, in statistics, 
one is mainly interested in the minimizer
and, accordingly, an
estimator $S_n$ should be consistent in the sense that
$S_n$ converges to $f_0$. For statistical inference, it is also
crucial to have estimates for the error of the
estimator, e.g., in order to obtain confidence sets or
hypothesis tests. While consistency results 
for the risk, e.g.\ \cite{steinwart2002},
\cite{zhang2004b}, \cite{steinwart2005},
and \cite{christmannsteinwart2007},
and also for the functions, 
e.g.\ \cite[\S\,3]{Steinwart:Christmann:2009} and
\cite[Cor.\ 3.7]{HableChristmann2011}, 
are well-known
for regularized kernel methods, 
there are only
very few and limited results concerning statistical
inference. In order to fill this gap,
asymptotic confidence sets for a
wide class of regularized kernel methods are developed
in the following. This is possible now because 
\cite{hable2012a} derives asymptotic normality of these methods
and, based on this result, 
estimating the error of the estimate gets tractable.
Let $f_{\mathbf{D}_{n},\Lambda_n}$ be the (nonparametric) estimate 
obtained by a regularized kernel method 
($\Lambda_n$ is a data-driven regularization parameter),
fix any $\lambda_0\in(0,\infty)$,
and let $\SVMplz$ be the minimizer of the regularized problem
\begin{eqnarray}\label{intro-regularized-problem}
  f\;\mapsto\;
  \mathcal{R}_{L,P}(f)
      \,+\,\lambda_0\|f\|_{H}^2
\end{eqnarray}
in the function space $H$ (a so-called reproducing kernel 
Hilbert space). According to
\cite[Theorem 3.1]{hable2012a}, 
under some assumptions,
the sequence of function-valued
random variables
$\sqrt{n}
    \big(f_{\mathbf{D}_{n},\Lambda_n}-\SVMplz
    \big)
$
weakly converges to a Gaussian process in $H$.
As a consequence, for differentiable functions $\psi:H\rightarrow\R^m$,
it follows that 
$$\sqrt{n}
    \Big(\psi\big(f_{\mathbf{D}_{n},\Lambda_{n}}
             \big)-
         \psi\big(\SVMplz\big)
    \Big)
    \;\;\leadsto\;\;\mathcal{N}_m(0,\Sigma_P)\;.
$$
In order to obtain asymptotic confidence sets, e.g., for
the vector of values $\SVMplz(x_j)$, $j\in\{1,\dots,m\}$,
one has to choose $\psi(f)=(f(x_1),\dots,f(x_m))$ and
it only remains to estimate the asymptotic covariance matrix $\Sigma_P$.
The derivation of a consistent estimator is not 
a trivial task and is the main issue of the article.
However, pointwise confidence sets for the true values of $\SVMplz$
are not the only possibility to directly apply the results of the
article. We also obtain confidence sets e.g.\ for
integrals of $\SVMplz$ (choose $\psi(f)=\int_B f\,d\lambda$)
or for the differential of $\SVMplz$ in a point $x_0$
(choose $\psi(f)=\partial f(x_0)$) and many others.
Essentially, it is only needed that $\psi$ takes its values in
$\R^m$ for any $m\in\N$ and is suitably differentiable. 

Note that we are only able to derive asymptotic confidence sets for
the unknown solution $\SVMplz$ of the regularized problem
(\ref{intro-regularized-problem}). 
Of course, it would be desirable to obtain asymptotic confidence
sets for the minimizer of the unregularized 
risk $\mathcal{R}_{L,P}$.
However, 
in our completely nonparametric setting ($P$ is totally unknown), 
this would require
a uniform rate of convergence of the learning algorithm/estimator
to the minimizer of $\mathcal{R}_{L,P}$ 
(if such a minimizer exists at all) and it is well-known 
from the no-free-lunch theorem \cite{devroye1982} that such a
uniform rate of convergence does not exist. 
That is, similar results for the
minimizer of the unregularized $\mathcal{R}_{L,P}$ can only
be obtained under substantial assumptions on the unknown
distribution $P$.\\
Accordingly, the approach in the present article which focuses
on applications in statistical inference 
considerably differs from the 
approach common in machine learning 
theory which focuses on (as fast as possible)
rates of convergence of the risk, e.g.,
\cite{steinwartscovel2007}, \cite{caponnetto2007},
\cite{blanchard2008}, \cite{steinwarthush2009},
\cite{mendelson2010}. This approach considers
large classes $\mathcal{P}$ of probability measures
for which learning rates, e.g., in the form
$$P^n\Big(
     \mathcal{R}_{L,P}(f_{D_n,\lambda_n})
     \,-\,\inf_f\mathcal{R}_{L,P}(f)
     \,\leq\,c_{P,\delta}\cdot n^{-\beta}
     \Big)
     \geq 1-\delta\;,
$$
exist and
where the rate of convergence $\beta>0$ does not depend on $P$
and the infimum is taken over all measurable functions
$f:\X\rightarrow\R$.   
Such learning rates are an important 
tool in order to compare theoretical
properties of
different learning algorithms. However, these results cannot be applied
offhand for statistical inference in real applications because
the constant $c_{P,\delta}$ is usually unknown. 
Furthermore, the focus lies on maximizing $\beta$ 
which, typically, results
in an increase of the constant $c_{P,\delta}$ so that
the bound $c_{P,\delta}\cdot n^{-\beta}$ might be large for 
ordinary sample
sizes $n$. In addition, whether a probability measure
belongs to $\mathcal{P}$ is often subject to assumptions which
are hard to communicate to practitioners and
to be satisfactorily checked 
or made plausible in applications. 
A common assumption is, e.g., Tsybakov's noise assumption 
\cite[p.\ 138]{Tsybakov2004}.

\bigskip

The present article derives asymptotic confidence sets 
for $\psi(\SVMplz)$ based on the asymptotic normality results 
of \cite{hable2012a}. 
So far, there are only very few publications which
are concerned with statistical inference for regularized kernel
methods.
In the special case of classification
by use of the hinge loss and linear SVMs (i.e. linear kernel),
asymptotic normality of the coefficients of the linear SVM
is shown in \cite{koo2008} under a number of regularity assumptions
(e.g.\ existence of continuous densities). 
Though this could
yield an alternative way of deriving asymptotic confidence sets
in this special case, this has not been done so far.
In the special case of classification
by use of the hinge loss and SVMs with finite-dimensional kernels
(i.e.\ a parametric setting),
\cite{jiang2008} shows asymptotic normality 
of the prediction error estimators and derive
confidence intervals for the prediction error of the
empirical SVM. 
In the special case of regression by use of
least-squares support vector regression,
\cite{brabanter2011} proposes approximate confidence 
intervals for the regression function whose 
derivation is partly based on heuristics;
it is not documented
whether these intervals approximately hold the intended
confidence level in simulated examples.

\bigskip

In the following Section \ref{subsec-setup}, some basics
of regularized kernel methods are recalled. The main
part of the article, Section \ref{sec-asymptotic-confidence},
consists of two subsections:
Subsection \ref{subsec-theory}
derives an asymptotically consistent estimator of 
$\Sigma_P$ and asymptotic confidence intervals; Subsection
\ref{subsec-computation} shows how the calculation
of the estimator can be done in a computationally tractable way. 
All proofs are given in the appendix.

\section{Regularized Kernel Methods}\label{subsec-setup}

Let $(\Omega,{\cal A},Q)$ be a probability space, 
let $\mathcal{X}$ be a closed and bounded subset of
$\R^d$,
and let $\mathcal{Y}$ be a 
closed subset of $\mathds{R}$ with Borel-$\sigma$-algebra
$\mathfrak{B}(\mathcal{Y})$\,. The Borel-$\sigma$-algebra
of $\mathcal{X}\times\mathcal{Y}$ is denoted by
$\mathfrak{B}(\mathcal{X}\times\mathcal{Y})$.
Let
\begin{eqnarray*}
  X_{1},\dots,X_{n}\;:\;\;(\Omega,{\cal A},Q)
  \;\longrightarrow\;\big(\mathcal{X},\mathfrak{B}(\mathcal{X})\big)\,, \\
  Y_{1},\dots,Y_{n}\;:\;\;(\Omega,{\cal A},Q)
  \;\longrightarrow\;\big(\mathcal{Y},\mathfrak{B}(\mathcal{Y})\big)\;\;
\end{eqnarray*}
be random variables such that
$\,(X_{1},Y_{1}),\dots,(X_{n},Y_{n})\,$ are independent 
and identically distributed according to some unknown
probability measure
$P$ on
$\big(\mathcal{X}\times\mathcal{Y},
  \mathfrak{B}(\mathcal{X}\times\mathcal{Y})
 \big)
$.
Define
$$\mathbf{D}_{n}\;:=\;\big((X_{1},Y_{1}),\dots,(X_{n},Y_{n})\big)
  \qquad\forall\,n\in\mathds{N}\;.
$$

A measurable map
$\,L:\mathcal{X}\times\mathcal{Y}\times\mathds{R}\rightarrow[0,\infty)\,
$
is called \emph{loss function}. A loss function $L$ is called
\emph{convex} loss function if it is convex in its third argument, 
i.e.
$t\mapsto L(x,y,t)$ is convex for every $(x,y)\in\XY$. Furthermore,
a loss function $L$ is called $P$-integrable Nemitski loss function
of order $p\in[1,\infty)$ if there is a $P$-integrable function
$b:\XY\rightarrow\R$ 
and a constant $c\in(0,\infty)$
such that
$$\big|L(x,y,t)\big|\;\leq\;b(x,y)+c|t|^{p}
 \qquad\forall\,(x,y,t)\in\XY\times\R\;.
$$
If $b$ is even $P$-\emph{square}-integrable, $L$ is called
$P$-\emph{square}-integrable Nemitski loss function
of order $p\in[1,\infty)$.
The \emph{risk} of a measurable function 
$f:\mathcal{X}\rightarrow\mathds{R}$
is defined by
$$\mathcal{R}_{L,P}(f)\;=\;
  \int_{\mathcal{X}\times\mathcal{Y}}L\big(x,y,f(x)\big)\,
  P\big(d(x,y)\big)\;.
$$
The goal is to estimate a function $f:\mathcal{X}\rightarrow\R$
which minimizes this risk. 
The estimates obtained from regularized kernel methods
are elements of so-called reproducing kernel Hilbert spaces
(RKHS) $H$. An RKHS $H$ is a certain Hilbert space 
of functions $f:\mathcal{X}\rightarrow\mathds{R}$ which is
generated by a \emph{kernel}  
$k:\mathcal{X}\times\mathcal{X}\rightarrow\mathds{R}$\,.
See e.g.\ \cite{schoelkopf2002}, \cite{berlinet2004},
\cite{steinwart2008},
or \cite{Hofmann:Schoelkopf:2008}
for details about these concepts.

Let $H$ be such an RKHS. Then, the
\emph{regularized risk} of an element $f\in H$ is defined to be
$$\mathcal{R}_{L,P,\lambda}(f)
  \;=\;\mathcal{R}_{L,P}(f)\,+\,\lambda\|f\|_{H}^2\;,
  \qquad\text{where}\;\;\;\lambda\in(0,\infty)\,.
$$

An element $f\in H$ is denoted by $f_{P,\lambda}$ if it
minimizes the regularized risk in $H$\,. That is,
$$\mathcal{R}_{L,P}(f_{P,\lambda})
      \,+\,\lambda\|f_{P,\lambda}\|_{H}^2\;=\;
  \inf_{f\in H}\,\big(\mathcal{R}_{L,P}(f)\,+\,\lambda\|f\|_{H}^2
                 \big)\;.
$$
The estimator is defined by
$$S_{n}\;:\;\;(\mathcal{X}\times\mathcal{Y})^n\times(0,\infty)
  \;\rightarrow\;H\,,\qquad
  (D_{n},\lambda)\;\mapsto\;
  f_{D_{n},\lambda}
$$
where $f_{D_{n},\lambda}$ is that function $f\in H$ which minimizes
\begin{eqnarray}\label{setup-regularized-empirical-risk}
  \frac{1}{n}\sum_{i=1}^{n}L\big(x_{i},y_{i},f(x_{i})\big)
  \,+\,\lambda\|f\|_{H}^{2}
\end{eqnarray}
in $H$ for 
$D_{n}=((x_{1},x_{2}),\dots,(x_{n},y_{n}))\,\in\,
 (\mathcal{X}\times\mathcal{Y})^{n}
$\,.
The estimate $f_{D_{n},\lambda}$ 
uniquely exists for every $\lambda\in(0,\infty)$ and every
data-set $D_{n}\in(\XY)^{n}$
if $t\mapsto L(x,y,t)$ is convex for every $(x,y)\in\XY$.

\smallskip

In the article, the symbol $\leadsto$ denotes
weak convergence of probability measures
or random variables.

\section{Asymptotic Confidence Intervals}\label{sec-asymptotic-confidence}

\subsection{Theory}\label{subsec-theory}

The derivation of asymptotic confidence sets is based on
the result in \cite{hable2012a} that, under some assumptions,
$$\sqrt{n}
    \big(f_{\mathbf{D}_{n},\Lambda_n}-\SVMplz
    \big)
    \;\;\leadsto\;\;\mathds{H}_P
    \qquad \text{in}\;\;H
$$
where $\mathds{H}_P$ is a mean-zero Gaussian process in $H$
and $\Lambda_n$ is a random regularization parameter
(e.g.\ data-driven).
Therefore, the same assumptions as in 
\cite{hable2012a} are needed; they are
collocated in the following:
\begin{assumption}\label{basic-assumptions}
  Let $\mathcal{X}\subset\R^d$ be closed and bounded and
  let $\mathcal{Y}\subset\R$ be closed. 
  Assume that $k:\XX\rightarrow\R$ is the restriction of an
  $r$\,-\,times continuously differentiable kernel 
  $\tilde{k}:\R^d\times\R^d\rightarrow\R$
  such that
  $r>d/2$ and $k\not=0$.
  Let $H$ be the RKHS of $k$
  and let $P$ be a probability measure on $(\XY,\mathfrak{B}(\XY))$\,.
  Let 
  $$L\;:\;\;\XY\times\R\;\rightarrow\;[0,\infty)\,,\qquad
    (x,y,t)\;\mapsto\;L(x,y,t)
  $$
  be a convex, $P$-square-integrable Nemitski loss function 
  of order $p\in[1,\infty)$
  such that
  the partial derivatives
  $$\Ls(x,y,t)\;:=\;\frac{\partial L}{\partial t}(x,y,t)
    \qquad\text{and}\qquad
    \Lss(x,y,t)\;:=\;\frac{\partial^2 L}{\partial^2 t}(x,y,t)
  $$
  exist for every $(x,y,t)\in\XY\times\R$\,.
  Assume that the maps
  $$(x,y,t)\;\mapsto\;\Ls(x,y,t) \qquad\text{and}\qquad
    (x,y,t)\;\mapsto\;\Lss(x,y,t)
  $$
  are continuous. Furthermore, assume that for every $a\in(0,\infty)$,
  there is a $b_{a}^{\prime}\in L_{2}(P)$ and a constant
  $b_{a}^{\prime\prime}\in [0,\infty)$ such that,
  for every $(x,y)\in\XY$,
  \begin{eqnarray}\label{theorem-sqrt-n-consistency-1}
    \sup_{t\in[-a,a]}\big|\Ls(x,y,t)\big|\;\leq\;b_{a}^{\prime}(x,y)
    \quad\;\text{and}\quad\;
    \sup_{t\in[-a,a]}\big|\Lss(x,y,t)\big|\;\leq\;
         b_{a}^{\prime\prime}\;.
  \end{eqnarray}
\end{assumption}

\medskip

These assumptions are relatively mild. 
In particular, the assumptions on
$k$ are fulfilled for all of the most common kernels
(e.g.\ Gaussian RBF kernel, polynomial kernel, exponential kernel, 
linear kernel). Though assuming differentiability 
of the loss function
is an obvious
restriction (as it does not cover 
some of the most popular loss functions
as hinge, epsilon-insensitive, and pinball), 
this assumption is not based on
any unknown entity such as the model distribution $P$\,.
Therefore, a practitioner can a priori
meet this requirement by a suitable choice of the
loss function (e.g.\ the least-squares loss
for regression, the logistic loss for classification
(or smoothed versions of hinge, epsilon-insensitive, and pinball).
This is contrary to the assumptions common
in order to establish rates of convergence to the infimal risk.
Typically, the assumptions used there depend on the unknown $P$
so that they can hardly be checked in applications
and are mathematically involved so that they 
can hardly be
communicated to practitioners.
In Assumption 
\ref{basic-assumptions}, the only assumptions on $P$ 
are integrability assumptions,
which are natural as such assumptions are 
necessary even for ordinary central limit theorems.
Explicit examples where Assumption 
\ref{basic-assumptions} is fulfilled are given in
Section \ref{sec-applications}.

\medskip

Under these assumptions, we have asymptotic normality:
\begin{theorem}{\cite[Theorem 3.1]{hable2012a}}
\label{theorem-sqrt-n-consistency}
  Let Assumption \ref{basic-assumptions} be fulfilled. Then, 
  for every $\lambda_{0}\in(0,\infty)$,
  there is a tight, Borel-measurable
  Gaussian process 
  $$\mathds{H}_P\;:\;\;\Omega\;\rightarrow\;H\,,\quad\;
    \omega\;\mapsto\;\mathds{H}_P(\omega)
  $$ 
  such that, 
  \begin{eqnarray}\label{theorem-sqrt-n-consistency-2}
    \sqrt{n}
    \big(f_{\mathbf{D}_{n},\Lambda_{n}}-\SVMplz
    \big)
    \;\;\leadsto\;\;\mathds{H}_P
    \qquad \text{in}\;\;H
  \end{eqnarray}
  for every Borel-measurable sequence
  of random regularization parameters $\Lambda_{n}$ with
  $$\sqrt{n}\big(\Lambda_{n}-\lambda_0\big)
    \;\xrightarrow[\;n\rightarrow\infty\;]{}\;0
    \qquad\text{in probability\,.}
  $$
  The Gaussian process $\mathds{H}_P$ 
  is zero-mean; i.e.,
  $\mathds{E}\la f,\mathds{H}_P\ra=0$ for every $f\in H$\,.
\end{theorem}

Recall that
a map $\psi\,:\,\,H\rightarrow\R^m$ is Hadamard differentiable at
some $f_0\in H$ if and only if there exists a 
$\psi^\prime_{f_0}
 =(\psi^\prime_{f_0,1},\dots,\psi^\prime_{f_0,m})\in H^m
$
such that, for every 
sequence $t_\ell\searrow0$ in $\R$, and for every sequence
$h_{\ell}\rightarrow h$ in $H$,
$$\lim_{\ell\rightarrow\infty}
  \left|\frac{\psi(f_0+t_\ell h_{\ell})-\psi(f_0)}{t_\ell}
         -\big\la \psi^\prime_{f_0},h\big\ra
  \right| 
  \;=\;0\;.
$$
The element $\psi^\prime_{f_0}\in H^m$ is called derivative of
$\psi$ at $f_0$. For $h\in H$ and $\psi^\prime_{f_0}
 =(\psi^\prime_{f_0,1},\dots,\psi^\prime_{f_0,m})\in H^m
$,
the expression $\la \psi^\prime_{f_0},h\ra$ denotes the element
of $\R^m$ whose components are given by 
$\la \psi^\prime_{f_0,j},h\ra$, $j\in\{1,\dots,m\}$. 
 
By a routine application of the functional
delta method \cite[Theorem 3.9.4]{vandervaartwellner1996}, 
we get the following corollary: 

\begin{corollary}\label{cor-finite-dim-asymptotic-normality}
  Let Assumption \ref{basic-assumptions} be fulfilled,
  let $\lambda_{0}\in(0,\infty)$,
  and let $\psi\,:\,\,H\rightarrow\R^m$ be Hadamard-differentiable
  in $\SVMplz$ with derivative $\psi^\prime_{\SVMplz}$.
  Then, there is a covariance matrix $\Sigma_P\in\R^{m\times m}$
  such that,
  for every Borel-measurable sequence
  of random regularization parameters $\Lambda_{n}$ with
  $$\sqrt{n}\big(\Lambda_{n}-\lambda_0\big)
    \;\xrightarrow[\;n\rightarrow\infty\;]{}\;0
    \qquad\text{in probability\,,}
  $$
  it holds that
  $$\sqrt{n}
    \Big(\psi\big(f_{\mathbf{D}_{n},\Lambda_{n}}
             \big)-
         \psi\big(\SVMplz\big)
    \Big)
    \;\;\leadsto\;\;\mathcal{N}_m(0,\Sigma_P)\;.
  $$
  The limit $\mathcal{N}_m(0,\Sigma_P)$ is equal to
  the distribution of 
  $\big\la\psi^\prime_{\SVMplz},\mathds{H}_P\big\ra$
  where $\mathds{H}_P$ is given by (\ref{theorem-sqrt-n-consistency-2}).
\end{corollary}

Accordingly, in order to derive 
asymptotic confidence intervals, the main issue which
remains to be solved is to
calculate or rather 
consistently estimate
the covariance matrix $\Sigma_P$ . In principle, $\Sigma_P$
is completely known if $P$ is known -- as can be seen from the
proof of Theorem \ref{theorem-sqrt-n-consistency} given in 
\cite{hable2012a}.
This suggests
to estimate $\Sigma_P$ by a plug-in estimator where
$P$ is replaced by the empirical measure $\mathds{P}_{\mathbf{D}_n}$. 
However, this is a challenging 
task because $\mathds{H}_P$ is given by
$\mathds{H}_P=S_P^\prime(\mathds{G}_P)$ where
$S_P^\prime$ is a (complicated) continuous linear operator and
$\mathds{G}_P$ is a random variable which takes its values 
in a large function space. Hence, calculating 
$\Sigma_P=
 \text{Cov}\big(\big\la\psi^\prime_{\SVMplz},\mathds{H}_P\big\ra\big)
$
means to calculate an integral with respect to a measure on
that function space. Fortunately, this can be avoided
as follows from Prop.\ \ref{prop-covariance-matrix}.
There, $\Sigma_P$ is specified in a way which is more accessible
to a plug-in estimator.
The
consistency of the resulting plug-in estimator is given in
Theorem \ref{theorem-consistency-covariance-estimator}.
Note that $\Sigma_P$ can be degenerated to 0 in Corollary
\ref{cor-finite-dim-asymptotic-normality}. In order to 
derive asymptotic confidence sets,
degeneracy has to be excluded by adding additional assumptions
(Assumption \ref{assumptions-non-deg-of-the-limit-marginals}) below. 

\begin{proposition}\label{prop-covariance-matrix}
  Let Assumption \ref{basic-assumptions} be fulfilled,
  let $\lambda_{0}\in(0,\infty)$,
  and let $\psi\,:\,\,H\rightarrow\R^m$ be Hadamard-differentiable
  in $\SVMplz$ with derivative $\psi^\prime_{\SVMplz}$.
  Define 
  $$g_{P,\lambda_0}:\;\;
    \XY\;\rightarrow\;\R\,,\quad
    (x,y)\;\mapsto\; 
    -\Ls\big(x,y,\SVMplz(x)\big)
    \big\la\psi^\prime_{\SVMplz},K_{P}^{-1}\big(\Phi(x)\big)\big\ra
  $$
  where $K_{P}$ denotes the continuous linear operator
  defined in (\ref{derivative-svm-functional-K-component-app}).
  Then, the covariance matrix $\Sigma_P$ in
  Corollary \ref{cor-finite-dim-asymptotic-normality}
  is equal to
  \begin{eqnarray}\label{prop-covariance-matrix-2}
    \Sigma_P\;=\;
    \textup{Cov}
      \big(g_{P,\lambda_0}(X_1,Y_1)
      \big)\;.
  \end{eqnarray}
\end{proposition}
It follows from Prop.\ \ref{prop-covariance-matrix} that
$\Sigma_P$ could be estimated by the standard covariance estimator
for the $\R^m$-valued i.i.d.\ random variables
$$g_{P,\lambda_0}(X_1,Y_1),\dots,g_{P,\lambda_0}(X_n,Y_n)
$$
if $P$ was known. However, as $P$ is unknown, we have to replace
$P$ by the empirical measure and 
$\psi^\prime_{\SVMplz}$ by an 
estimator $\psi^\prime_{\mathbf{D}_{n},\Lambda_n}$ of 
$\psi^\prime_{\SVMplz}$.
Then, we may estimate $\Sigma_P$
by the non-i.i.d.\ random variables
$$g_{\mathbf{D}_{n},\Lambda_{n}}(X_1,Y_1),\dots,
  g_{\mathbf{D}_{n},\Lambda_{n}}(X_n,Y_n)
$$
where 
\begin{eqnarray}\label{def-of-random-g-for-cov-estimator}
  g_{\mathbf{D}_{n},\Lambda_{n}}(x,y)\;=\;
  -\Ls\big(x,y,f_{\mathbf{D}_{n},\Lambda_{n}}(x)\big)
   \big\la\psi^\prime_{\mathbf{D}_{n},\Lambda_{n}},
           K_{\mathbf{D}_{n},\Lambda_n}^{-1}\big(\Phi(x)\big)
   \big\ra
\end{eqnarray}
and $K_{\mathbf{D}_{n}(\omega),\Lambda_n(\omega)}:\;H\,\rightarrow\,H$
is the continuous linear operator given by
\begin{eqnarray}\label{def-of-random-K}
  K_{\mathbf{D}_{n},\Lambda_n}(f)
  \,=\,
  2\Lambda_n f+
  \frac{1}{n}\!
  \sum_{i=1}^{n}\!
    \Lss\big(X_i,Y_i,f_{\mathbf{D}_{n},\Lambda_{n}}(X_i)
        \big)f(X_i)\Phi(X_i)
\end{eqnarray}
for every $f\in H$.   
The following theorem states that the resulting plug-in 
covariance estimator is strongly consistent. It is also shown
that the estimator is measurable. This is not obvious 
as the proof of Theorem \ref{theorem-sqrt-n-consistency}
is based on the theory of empirical processes
and the map 
$D_n\mapsto\mathds{P}_{D_n}
$
(which maps a set of data to its empirical measure
as an element of a certain function space)
is typically not Borel-measurable; see e.g.\ 
\cite[\S\,1.1]{vandervaartwellner1996}.

\begin{assumption}\label{assumption-estimator-of-hadamard-derivative}
  Let $\psi\,:\,\,H\rightarrow\R^m$ be Hadamard-differentiable
  at $\SVMplz$ with derivative $\psi^\prime_{\SVMplz}$
  and let $\psi^\prime_{\mathbf{D}_{n},\Lambda_{n}}$ be an
  estimator of $\psi^\prime_{\SVMplz}$ which is
  strongly consistent, i.e.,
  \begin{eqnarray}\label{assumption-estimator-of-hadamard-derivative-1}
    \big\|\psi^\prime_{\mathbf{D}_{n},\Lambda_{n}}-\psi^\prime_{\SVMplz}
    \big\|_{H^m}
    \;\xrightarrow[\;n\rightarrow\infty\;]{}\;0
    \qquad\text{almost surely}.
  \end{eqnarray}
\end{assumption}

\begin{theorem}\label{theorem-consistency-covariance-estimator}
  Let Assumption \ref{basic-assumptions} and Assumption
  \ref{assumption-estimator-of-hadamard-derivative}
  be fulfilled.
  Fix $\lambda_{0}\in(0,\infty)$
  and let $\Sigma_P\in\R^{m\times m}$ be the covariance matrix in
  Corollary \ref{cor-finite-dim-asymptotic-normality}. Then,
  for every Borel-measurable sequence
  of random regularization parameters $\Lambda_{n}$ with
  $$\sqrt{n}\big(\Lambda_{n}-\lambda_0\big)
    \;\xrightarrow[\;n\rightarrow\infty\;]{}\;0
    \qquad\text{almost surely\,,}
  $$
  the estimator
  $$\hat{\Sigma}_n(\mathbf{D}_n,\Lambda_n)\;=\;
    \frac{1}{n}\sum_{i=1}^n
      \Big(\tilde{g}_{\mathbf{D}_{n},\Lambda_{n}}(X_i,Y_i)
      \Big)\cdot
      \Big(\tilde{g}_{\mathbf{D}_{n},\Lambda_{n}}(X_i,Y_i)
      \Big)^{\mathsf{\scriptscriptstyle T}}
  $$
  with
  $$\tilde{g}_{\mathbf{D}_{n},\Lambda_{n}}\!(X_i,Y_i)
    \,:=\,
    g_{\mathbf{D}_{n},\Lambda_{n}}\!(X_i,Y_i)
            -\frac{1}{n}\sum_{j=1}^n
               g_{\mathbf{D}_{n},\Lambda_{n}}\!(X_j,Y_j)
    \quad\forall\,i\in\{1,\dots,n\}
  $$
  is measurable and strongly consistent, i.e.,
  $$\hat{\Sigma}_n(\mathbf{D}_n,\Lambda_n)
    \;\xrightarrow[\;n\rightarrow\infty\;]{}\;
    \Sigma_P
    \qquad\text{almost surely}.
  $$
\end{theorem}

The following remark specifies a natural candidate for 
an estimator of $\psi^\prime_{\SVMplz}$; the proof is given in
the appendix. 
\begin{remark}\label{remark-estimator-of-hadamard-derivative}
  If $\psi$ is Hadamard-differentiable
  at every $f\in H$ with derivative $\psi^\prime_f$ and if
  $f\mapsto\psi^\prime_f$ is continuous, then Assumption
  \ref{assumption-estimator-of-hadamard-derivative} is fulfilled
  for the estimator 
  $$\psi^\prime_{\mathbf{D}_{n},\Lambda_{n}}\;:=\;
    \psi^\prime_{f_{\mathbf{D}_{n},\Lambda_{n}}}
  $$
  -- provided that $\Lambda_{n}$ converges to $\lambda_0$ almost surely.
\end{remark}

The calculation of the estimator $\hat{\Sigma}_n(\mathbf{D}_n,\Lambda_n)$
for a given data set is an issue of its own because it is
burdened by the fact that we have to 
solve $n$ equations 
$$K_{\mathbf{D}_{n},\Lambda_{n}}(f_i)\;=\;\Phi(X_i)\,, 
  \qquad i\in\{1,\dots,n\},
$$
in the typically infinite dimensional function space $H$.
As we will see in Subsection 
\ref{subsec-computation} below, this problem
can satisfactorily be solved 
(Prop.\ \ref{prop-calculation-of-covariance-estimator}).
In fact, these equations can be solved jointly, essentially 
by calculating the Moore-Penrose pseudoinverse of an
$n\times n$-matrix once only.

\medskip

In order to derive asymptotic confidence intervals based on
Corollary \ref{cor-finite-dim-asymptotic-normality}, it is 
desirable that the covariance matrix 
$\Sigma_P$ has full rank. 
Lemma \ref{lemma-non-degeneracy-of-the-limit-marginals}
in the Appendix
yields that 
this can be achieved by
the following two weak conditions:

\begin{assumption}\label{assumptions-non-deg-of-the-limit-marginals}
  Assume that, for $P_{\X}(dx)$\,-\,a.e.\ $x\in\X$, 
  \begin{eqnarray}\label{non-degeneracy-of-the-limit-marginals-cond-1}
    \exists\,y_1,y_2\in\textup{supp}\big(P(dy|x)\big)
    \,\,\text{ s.t.\ }\,\,
    \Ls\big(x,y_1,\SVMplz(x)\big)\not=
    \Ls\big(x,y_2,\SVMplz(x)\big).\,
  \end{eqnarray}
  For every $j\in\{1,\dots,m\}$, let
  $\psi^\prime_{\SVMplz,j}\in H$ denote the $j$-th component
  of $\psi^\prime_{\SVMplz}$
  and assume that
  \begin{eqnarray}\label{non-degeneracy-of-the-limit-marginals-cond-2}
    \psi^\prime_{\SVMplz,1},\dots,\psi^\prime_{\SVMplz,m}\,\,
    \text{ are linearly independent on }\,
    \textup{supp}(P_{\X})\;.
  \end{eqnarray}
\end{assumption}
Due to continuity,
Assumption (\ref{non-degeneracy-of-the-limit-marginals-cond-2})
can be reformulated to the following condition: 
\begin{eqnarray}\label{non-degeneracy-of-the-limit-marginals-cond-3}
    a\Transp \psi^\prime_{\SVMplz}\,=\,
    0\;\;\;P_{\X}\text{-a.s.}
    \,\text{ for some }a\in\R^m
    \qquad\Rightarrow\qquad
    a=0\;.
\end{eqnarray}
As we will see in the examples in the applications section,
Assumption \ref{assumptions-non-deg-of-the-limit-marginals} indeed
provides weak and simple conditions. 
E.g., in case of the least-squares loss
or the logistic loss, Assumption 
(\ref{non-degeneracy-of-the-limit-marginals-cond-1})
is equivalent to assuming that $P(dy|x)$ is not a Dirac measure.

\medskip

From the above results and assumptions, it follows that
$$\sqrt{n}\cdot\hat{\Sigma}_n(\mathbf{D}_n,\Lambda_n)^{-\frac{1}{2}}
    \Big(\psi\big(f_{\mathbf{D}_{n},\Lambda_{n}}
             \big)-
         \psi\big(\SVMplz\big)
    \Big)
    \;\;\leadsto\;\;\mathcal{N}_m\big(0,\textup{Id}_{m\times m}\big)
$$
so that we get elliptical confidence sets which are
asymptotically correct:
\begin{theorem}\label{theorem-asymptotic-confidence-sets}
  Let $\lambda_{0}\in(0,\infty)$ and
  let Assumption \ref{basic-assumptions}, Assumption
  \ref{assumption-estimator-of-hadamard-derivative}, 
  and Assumption \ref{assumptions-non-deg-of-the-limit-marginals}
  be fulfilled. 
  Let $\Lambda_{n}$ be a sequence
  of Borel-measurable random regularization parameters  with
  $$\sqrt{n}\big(\Lambda_{n}-\lambda_0\big)
    \;\xrightarrow[\;n\rightarrow\infty\;]{}\;0
    \qquad\text{almost surely}\;.
  $$
  Fix any $\alpha\in(0,1)$, let $\chi_{m,\alpha}^2$
  be the $(1-\alpha)$th quantile of the
  chi-squared distribution with $m$ degrees of freedom and
  $$C_{n,\alpha}(\mathbf{D}_n,\Lambda_n)
    :=
    \Big\{w\in\R^m
    \Big|\,\,\big\|
             \hat{\Sigma}_n(\mathbf{D}_n,\Lambda_n)^{-\frac{1}{2}}
             \big(w-\psi(f_{\mathbf{D}_{n},\Lambda_{n}})
             \big)
             \big\|^2_{\R^m}
             \leq \, \tfrac{\chi_{m,\alpha}^2}{n}
    \Big\}.
  $$
  Then,
  $$Q\Big(\psi\big(\SVMplz\big)
           \,\in\,
          C_{n,\alpha}(\mathbf{D}_n,\Lambda_n)
     \Big)
    \;\xrightarrow[\;n\rightarrow\infty\;]{}\;1-\alpha\;.
  $$
\end{theorem}
Note that the confidence set $C_{n,\alpha}(\mathbf{D}_n,\Lambda_n)$
is an ellipsoid in $\R^m$ 
which is centered at $\psi(f_{\mathbf{D}_{n},\Lambda_{n}})$
and
whose principal axes are given by
$$\sqrt{\frac{\chi_{m,\alpha}^2\gamma_1}{n}\,}\cdot v_1\,,
  \,\dots,\,\sqrt{\frac{\chi_{m,\alpha}^2\gamma_m}{n}\,}\cdot v_m
$$
where $\gamma_1,\dots,\gamma_m$ are the eigenvalues
and $v_1,\dots,v_m$ are corresponding orthonormal
eigenvectors of the matrix $\hat{\Sigma}_n(\mathbf{D}_n,\Lambda_n)$.

\subsection{Computation of Asymptotic Confidence Sets} 
   \label{subsec-computation}

The calculation of the estimator 
$\hat{\Sigma}_n(\mathbf{D}_n,\Lambda_n)$
for a given data set is burdened by the fact that we have to 
solve every of the following $n$ equations 
$$K_{\mathbf{D}_{n},\Lambda_{n}}(f_i)\;=\;\Phi(X_i)\,, 
  \qquad i\in\{1,\dots,n\},
$$
in the typically infinite dimensional function space $H$.
In particular for a large sample size $n$, this seems to
be problematic. However, 
Prop.\ \ref{prop-calculation-of-covariance-estimator} 
below yields that the problem
can essentially be reduced to the calculation of a single 
Moore-Penrose pseudoinverse of an $n\times n$-matrix after the 
following preparation:
Let $D_n=\big((x_1,y_1),\dots,(x_n,y_n)\big)\in(\XY)^n$. 
Then, there is
always a maximal subset
$\{\Phi(x_{i_1}),\dots,\Phi(x_{i_r})\}$
of $\{\Phi(x_1),\dots,\Phi(x_n)\}$ such that
$\Phi(x_{i_1}),\dots,\Phi(x_{i_r})$ are linearly independent --
i.e.\ $\{\Phi(x_{i_1}),\dots,\Phi(x_{i_r})\}$ is a basis
of the vector space spanned by $\Phi(x_1),\dots,\Phi(x_n)$.
Accordingly, for every $i\in\{1,\dots,n\}$,
there are $\beta_{1i},\dots,\beta_{ri}\,\in\,\R$ such that
\begin{eqnarray}\label{prop-calculation-of-covariance-estimator-prep-1}
  \Phi(x_i)=\sum_{j=1}^r \beta_{ji}\Phi(x_{i_j})\;. 
\end{eqnarray}
Define
\begin{eqnarray}\label{prop-calculation-of-covariance-estimator-prep-2}
  B_{D_n}=
    \left(
      \begin{array}{ccc}
        \beta_{11} & \hdots & \beta_{1n} \\
        \vdots & \ddots & \vdots \\
        \beta_{r1} & \hdots & \beta_{rn} \\
      \end{array}
    \right)
  \;\in\;\R^{r\times n}\;.
\end{eqnarray}
E.g., in case of a Gaussian RBF kernel, vectors
$\Phi(x_{i_1}),\dots,\Phi(x_{i_r})$ are linearly independent if
and only if all $x_{i_1},\dots,x_{i_r}$ differ; see e.g.\ 
\cite[Theorem 2.18]{schoelkopf2002}.
Hence, in this case, finding $B_{D_n}$ only means to
identify all ties in the covariates -- and, if 
there are no such ties, $B_{D_n}$ is just the
$n\times n$\,-\,identity matrix.

\begin{proposition}\label{prop-calculation-of-covariance-estimator}
  Let Assumption \ref{basic-assumptions} be fulfilled.
  Fix any set of data
  $D_n=\big((x_1,y_1),\dots,(x_n,y_n)\big)\in(\XY)^n$
  and any $\lambda\in(0,\infty)$. 
  Define $B_{D_n}$ according to 
  (\ref{prop-calculation-of-covariance-estimator-prep-1})
  and (\ref{prop-calculation-of-covariance-estimator-prep-2}).
  Let 
  $\Lss_{D_n,\lambda}\in\R^{n\times n}$ 
  denote the
  diagonal matrix with diagonal entries
  $$\Lss\big(x_1,y_1,f_{D_{n},\lambda}(x_1)
        \big)
    ,\dots,
    \Lss\big(x_n,y_n,f_{D_{n},\lambda}(x_n)
        \big),
  $$
  define the $n\times n$-matrix
  $$A_{D_n,\lambda}\;=\;
    2\lambda\cdot\textup{Id}_{n\times n}\,+\,\frac{1}{n}\cdot
    \Lss_{D_n,\lambda}\cdot
    \left(
      \begin{array}{ccc}
        k(x_1,x_1) & \hdots & k(x_1,x_n) \\
        \vdots & \ddots & \vdots \\
        k(x_n,x_1) & \hdots & k(x_n,x_n) \\
      \end{array}
    \right),
  $$ 
  and let $(B_{D_n}A_{D_n,\lambda})^-$ be the 
  Moore-Penrose pseudoinverse of $B_{D_n}A_{D_n,\lambda}$. 
  Then, for
  every $x\in\X$ and $y\in\Y$, 
  $$K_{D_{n},\lambda}^{-1}\big(\Phi(x)\big)\;=\;
    \frac{1}{2\lambda}\Phi(x)+
    \sum_{i=1}^n\alpha_i(x)\Phi(x_i)
  $$
  and
  $$g_{D_{n},\lambda}(x,y)\,=\,-
    \Ls\big(x,y,f_{D_{n},\lambda}(x)\big)
    \!\cdot\!
    \bigg(\frac{1}{2\lambda}
            \psi_{f_{D_{n},\lambda}}^\prime(x)+
          \sum_{i=1}^n\!\alpha_i(x)
            \psi_{f_{D_{n},\lambda}}^\prime(x_i)\!\!
    \bigg)
  $$
  where
  $$\left(\!\!\!
      \begin{array}{c}
        \alpha_1(x) \\ \vdots \\ \alpha_n(x)
      \end{array}\!\!\!
    \right)
    \;=\;
    -\frac{1}{2 n\lambda}\cdot (B_{D_n}A_{D_n,\lambda})^- B_{D_n}
    \!\cdot\!
    \left(\!\!\!
      \begin{array}{c}
        \Lss\big(x_1,y_1,f_{D_{n},\lambda}(x_1)
        \big)
        k(x_1,x)  \\ 
        \vdots \\ 
        \Lss\big(x_n,x_n,f_{D_{n},\lambda}(x_n)
        \big)
        k(x_n,x)   
      \end{array}\!\!\!
    \right).
  $$
\end{proposition}

\medskip

By use of this proposition, the calculation of 
the estimator 
$\hat{\Sigma}_n(D_n,\lambda)$
is unproblematic. 
According to its definition, it is enough to calculate
the values $g_{D_n,\lambda}(x_i,y_i)$, $i\in\{1,\dots,n\}$,
and, in order to do this, the matrices 
$B_{D_n}$ and $A_{D_n,\lambda}$ have to be defined 
and $(B_{D_n}A_{D_n,\lambda})^-$ has to be calculated
once only. Then, all values 
$g_{D_n,\lambda}(x_i,y_i)$, $i\in\{1,\dots,n\}$, 
can simultaneously be calculated by matrix calculus.
After that, it only remains to calculate 
the inverse of the matrix $\hat{\Sigma}_n(D_n,\lambda)$
in order to calculate the elliptical confidence set.
In order to obtain the principal axes of the ellipse,
one only has to calculate an (orthonormal) eigendecomposition of
$\hat{\Sigma}_n(D_n,\lambda)$ instead.

\section{Applications}\label{sec-applications}

In Subsection \ref{subsec-theory}, a general scheme is developed 
how to derive asymptotic confidence sets for values
$\psi(\SVMplz)$ of
functionals $\psi:H\rightarrow\R^m$. This general scheme is exemplified
in a few possible applications from which it can also be seen that
the assumptions made in Subsection \ref{subsec-theory} are moderate
and, equally important, not mathematically involved
so that they are comprehensible to
practitioners.

\medskip

\textit{The input and the output space.}
  Let $\mathcal{X}\subset\R^d$ be closed and bounded and
  let $\mathcal{Y}\subset\R$ be closed. That is,
  the setting covers regression with 
  $\Y=\R$ and classification with $\Y=\{-1,+1\}$ as well.
  
\medskip

\textit{The kernel $k$.} 
  Let $\tilde{k}:\R^d\times\R^d\rightarrow\R$ be a kernel
  which is $r$\,-\,times continuously differentiable kernel
  where
  $r>d/2$.
  Let $k:\XX\rightarrow\R$ be the restriction of $\tilde{k}$
  on $\XX$. Let $k\not=0$. That is, every of the
  most common kernels can be chosen: 
  a Gaussian RBF kernel, a polynomial kernel, the linear kernel,
  the exponential kernel, or sums and products
  of such kernels. 
  
\medskip

\textit{The loss function $L$.} We exemplarily
  consider the following three settings:
  \begin{itemize}
   \item[(A)] Regression with the least-squares loss:
     Let
     $$L(x,y,t)\;=\;(y-t)^2
       \qquad\forall\,(x,y,t)\,\in\,\XY\times\R
     $$
     and assume that $\mathbb{E}Y^4<\infty$.
   \item[(B)] Regression with the logistic loss:
     Fix a constant $\sigma>0$ and define
     $$L(x,y,t)\;=\;
       -\sigma\cdot
       \log\frac{4\exp\big(\frac{y-t}{\sigma}\big)
                }{\big(1+\exp\big(\frac{y-t}{\sigma}\big)\big)^2}
       \qquad\forall\,(x,y,t)\,\in\,\XY\times\R
     $$
     and assume that $\mathbb{E}Y^2<\infty$.
   \item[(C)] Classification:
     Let $\Y=\{-1,+1\}$ and choose the least-squares loss
     $$L(x,y,t)\;=\;(1-yt)^2
       \qquad\forall\,(x,y,t)\,\in\,\XY\times\R
     $$
     or the logistic loss
     $$L(x,y,t)\;=\;
       \log\big(1+\exp(y-t)\big)
       \qquad\forall\,(x,y,t)\,\in\,\XY\times\R\;.
     $$
  \end{itemize}
  In every of these settings, 
  Assumption \ref{basic-assumptions} is fulfilled. 
  Furthermore, (\ref{non-degeneracy-of-the-limit-marginals-cond-1})
  in Assumption \ref{assumptions-non-deg-of-the-limit-marginals}
  can be rewritten as
  \begin{eqnarray}\label{non-degeneracy-of-the-limit-marginals-cond-1-rew}
    \text{Var}(Y|\,x)\;\not=\;0
    \qquad\text{for }P_{\X}(dx)\text{\,-\,a.e.\ }x\in\X\,.
  \end{eqnarray}
  If $\text{Var}(Y|\,x)=0$ for some $x\in\X$, then 
  $Y$ is deterministically fixed by $X=x$. Of course,
  $\text{Var}(Y|\,x)=0$ for some $x\in\X$ can happen at most
  in case of heteroscedastic 
  (or even more complicated) error terms. In addition to
  (\ref{non-degeneracy-of-the-limit-marginals-cond-1}), 
  the only remaining assumption is Assumption
  (\ref{non-degeneracy-of-the-limit-marginals-cond-2}),
  which we have to take care of when choosing a
  functional $\psi$.
  
\medskip

\textit{The regularization parameter $\Lambda_n$}.
  The regularization parameter can be randomly chosen,
  e.g.\ by use of any data-driven method
  (cross validation etc.).
  The only requirement is to make sure that
  $\sqrt{n}\big(\Lambda_n-\lambda_0\big)\longrightarrow 0
  $ 
  almost surely for $n\rightarrow\infty$.
  A simple way to fulfill this condition
  for any data-driven method
  is to
  choose a (possibly large) constant $c\in(0,\infty)$ and
  to modify the method in such a way that
  it picks a value from $[\lambda_0\,,\,\lambda_0+c/\sqrt{n\ln(n)}\,]$.
  Note that it is 
  indeed possible to
  use the same data for choosing the regularization parameter
  as for building the final estimate - just as usually done by 
  practitioners, 
  e.g., when applying cross validation.

\medskip

\textit{The functional $\psi$.} With these choices
  and assumptions, 
  the asymptotic confidence set 
  (Theorem \ref{theorem-asymptotic-confidence-sets}) is valid for
  every functional $\psi:H\rightarrow\R^m$
  which is Hadamard-differentiable at $\SVMplz$ and 
  fulfills (\ref{non-degeneracy-of-the-limit-marginals-cond-2}).
  In the following, some concrete examples for $\psi$ are 
  listed or even worked out in detail. In most cases, 
  $\psi$ is continuous and linear so that
  the derivative $\psi^\prime_{\SVMplz}$ is exactly known
  as it does not depend on the
  unknown $\SVMplz$. If $\psi^\prime_{\SVMplz}$ is exactly known,
  then Assumption (\ref{non-degeneracy-of-the-limit-marginals-cond-2})
  can be checked in real applications by use of the following 
  ``test'': Define the $m\times n$-matrix
  $$\Psi :=\big(\psi^\prime_{\SVMplz}(x_1),\dots,
                \psi^\prime_{\SVMplz}(x_n)
           \big)
  $$
  where $x_1,\dots,x_n$ are the observed values of the input variables.
  If Assumption (\ref{non-degeneracy-of-the-limit-marginals-cond-2})
  is violated, then the probability that 
  $\Psi$ has rank $m$ (i.e.\ full rank for $n>m$) is equal to 0.
  (This follows from continuity of $\psi^\prime_{\SVMplz}$ and
  the fact that $P_{\X}(\textup{supp}(P_{\X}))=1$.)
  That is, if the observed $\Psi$ has full rank, one can assume
  that (\ref{non-degeneracy-of-the-limit-marginals-cond-2}) is fulfilled.

\medskip

\textbf{Example 1}: Pointwise confidence intervals

Fix some $x_1,\dots,x_m\,\in\,\X$ and define
$$\psi(f)\;=\;(f(x_1),\dots,f(x_m))^{\scriptscriptstyle\mathsf{T}}\,,
  \qquad f\in H\,.
$$
Since $\psi:H\rightarrow\R^m$ is continuous and linear, 
$\psi$ is continuously Hadamard-differentiable. The derivative
is given by
$$\psi_f^\prime\;=\;
  \big(\Phi(x_1),\dots,\Phi(x_m)\big)^{\scriptscriptstyle\mathsf{T}}
  \,\in\,H^m\,,\qquad f\in H\,.
$$
Condition (\ref{non-degeneracy-of-the-limit-marginals-cond-2})
can be checked as described above.
Since
$$\psi_f^\prime(x)\;=\;
  \big(k(x,x_1),\dots,k(x,x_m)\big)^{\scriptscriptstyle\mathsf{T}}
  \qquad\forall\,x\in\X\,,\;\;\;f\in H\,,
$$
it follows from Prop.\ \ref{prop-calculation-of-covariance-estimator} 
that
$$g_{\mathbf{D}_{n},\Lambda_n}\!(x,y)\,=\,-
    \Ls\big(x,y,f_{\mathbf{D}_{n},\Lambda_n}(x)\big)
    \!\cdot\!
    \left(\!\!\!
      \begin{array}{c}
        \frac{1}{2\Lambda_n}k(x,x_1)+\sum_{i=1}^n\alpha_i(x)k(X_i,x_1) \\
        \vdots\\
        \frac{1}{2\Lambda_n}k(x,x_m)+\sum_{i=1}^n\alpha_i(x)k(X_i,x_m)
      \end{array}\!\!\!
    \right)
$$
where the $\alpha_i(x)$, $i\in\{1,\dots,n\}$, are calculated
according to
Prop.\ \ref{prop-calculation-of-covariance-estimator}. 
Fix any $\alpha\in(0,1)$. Then, 
Theorem \ref{theorem-asymptotic-confidence-sets} says that
$$Q\Big(\big(\SVMplz(x_1),\dots,\SVMplz(x_m)\big)
        \;\in\;C_{n,\alpha}(\mathbf{D}_{n},\Lambda_{n})
   \Big)
    \;\xrightarrow[\;n\rightarrow\infty\;]{}\;1-\alpha\;.
$$
where
$C_{n,\alpha}(\mathbf{D}_{n},\Lambda_{n})$
is the elliptical confidence set as
defined in Theorem \ref{theorem-asymptotic-confidence-sets}.
\hfill$\Box$

\medskip

Due to the reproducing property \cite[Def.\ 4.18]{steinwart2008}, 
Example 1 is a special case of
the following example.

\medskip

\textbf{Example 2}: Confidence intervals for inner products

Fix some $h_1,\dots,h_m\,\in\,H$ which are 
linearly independent on the support of
$P_\X$ and
define
$$\psi(f)\;=\;
  \big(\la f,h_1\ra,\dots,\la f,h_m\ra
  \big)^{\scriptscriptstyle\mathsf{T}}\,,
  \qquad f\in H\,.
$$
Since $\psi:H\rightarrow\R^m$ is continuous and linear, 
$\psi$ is continuously Fr\'echet differentiable and the derivative
is given by
$$\psi_f^\prime\;=\;
 (h_1,\dots,h_m)^{\scriptscriptstyle\mathsf{T}}
  \,\in\,H^m\,,\qquad f\in H\,,
$$
and condition (\ref{non-degeneracy-of-the-limit-marginals-cond-2})
is fulfilled.
It follows from Prop.\ \ref{prop-calculation-of-covariance-estimator} 
that
$$g_{\mathbf{D}_{n},\Lambda_n}(x,y)\,=\,-
    \Ls\big(x,y,f_{\mathbf{D}_{n},\Lambda_n}(x)\big)
    \!\cdot\!
    \left(
      \begin{array}{c}
        \frac{1}{2\Lambda_n}h_1(x)+\sum_{i=1}^n\alpha_i(x)h_1(X_i) \\
        \vdots\\
        \frac{1}{2\Lambda_n}h_m(x)+\sum_{i=1}^n\alpha_i(x)h_m(X_i)
      \end{array}
    \right)
$$
where $\alpha_i(x)$, $i\in\{1,\dots,n\}$, are calculated
according to
Prop.\ \ref{prop-calculation-of-covariance-estimator}. 
Fix any $\alpha\in(0,1)$. Then, 
Theorem \ref{theorem-asymptotic-confidence-sets} says that
$$Q\Big(\big(\la\SVMplz,h_1\ra,\dots,\la\SVMplz,h_m\ra\big)
        \;\in\;C_{n,\alpha}(\mathbf{D}_{n},\Lambda_{n})
   \Big)
    \;\xrightarrow[\;n\rightarrow\infty\;]{}\;1-\alpha\;.
$$
where
$C_{n,\alpha}(\mathbf{D}_{n},\Lambda_{n})$
is the elliptical confidence set as
defined in Theorem \ref{theorem-asymptotic-confidence-sets}.
\hfill$\Box$

\medskip

\textbf{Example 3}: Confidence set for the gradient

Fix any $x_0$ in the interior of $\X$ and, 
for every $f\in H$, let
$$\psi(f)\;=\;\partial f(x_0)\;\in\;\R^d
$$
be the gradient vector of $f$ in $x_0$. According to 
\cite[p.\ 130ff]{steinwart2008},
the partial derivative of $f$ in $x_0$ with respect to the 
$j$-th coordinate
of $x$ is given by
$\partial_j f(x_0)=\big\la f,\partial_j\Phi(x_0)\big\ra$.
Hence, this is again
a special case of Example 2 and it follows 
that
$$\psi_f^\prime(x)\;=\;
  \frac{\partial}{\partial \tilde{x}}
  k(x,\tilde{x})\Big|_{\tilde{x}=x_0}
  \qquad\forall\,x\in\X\,,\;\;f\in H\,.
$$
Again, Assumption (\ref{non-degeneracy-of-the-limit-marginals-cond-2})
can be checked as described above.
\hfill$\Box$

\medskip

\textbf{Example 4}: Confidence set for integrals

Fix any Borel set $B\subset\X$ and, 
for every $f\in H$, define
$$\psi(f)\;=\;\int_B f\,dP_{\X}\;\in\;\R^d\;.
$$
This is again a special case of Example 2, the derivative is
given by
$$\psi^\prime_f(x)\;=\;\int_B k(x,\tilde{x})\,P_{\X}(d\tilde{x})
  \qquad\forall\,x\in\X\,,\;\;f\in H
$$
and Assumption (\ref{non-degeneracy-of-the-limit-marginals-cond-2})
can be checked as described above.
\hfill$\Box$

\medskip

\textbf{Example 5}: Confidence interval for the $H$-norm
and for the $L_2$-norm

The map
$$f\;\mapsto\;\psi(f)=\|f\|^2_H
$$
is continuously Hadamard differentiable
with derivative $\psi_f^\prime=2f$ at $f$; see 
e.g.\ \cite[Example 5.1.6(c)]{denkowski2003}.
Condition (\ref{non-degeneracy-of-the-limit-marginals-cond-2})
is fulfilled if $\SVMplz$ is not 
$P_\X$\,-\,almost surely equal to 0. Hence,
it is possible to construct a confidence interval
for $\|\SVMplz\|_H^2$
according to Theorem \ref{theorem-asymptotic-confidence-sets}
and, therefore, also for $\|\SVMplz\|_H$ by
taking square roots.\\
Similarly, for any $B\subset\R^d$, the map 
$$f\;\mapsto\;\psi(f)=\|f\|^2_{L_2(B,\lambda^d)}
  =\int_B \big(f(x)\big)^2\,dx
$$
is continuously Hadamard-differentiable and the derivative at any
$f\in H$
is equal to $\psi_f^\prime=\int_B 2f(x)\Phi(x)\,dx$
(this follows from \cite[Lemma 2.21]{steinwart2008} 
where $L(x,y,t)=t^2$). 
Again, Condition (\ref{non-degeneracy-of-the-limit-marginals-cond-2})
is fulfilled if $\SVMplz$ is not $P_\X$\,-\,almost surely 
equal to 0 on $B$. 
This can be shown by considering the RKHS
which consists of the restrictions of the elements $f\in H$
on $\text{supp}(P_{\X})$.
\hfill$\Box$

\medskip

Similarly, to Example 4,
the map $f\mapsto\|f-\SVMplz\|_H^2$ is continuously
differentiable so that one might be tempted to
apply $\psi(f)=\|f-\SVMplz\|_H^2$ in Theorem
\ref{theorem-asymptotic-confidence-sets}
in order to obtain a confidence band for 
the whole function $\SVMplz$ and not just for a finite number of points
as in Example 1. However, this is not possible because then
the derivative is given by $\psi_f^\prime=2(f-\SVMplz)$
so that $\psi_{\SVMplz}^\prime=0$ which violates
(\ref{non-degeneracy-of-the-limit-marginals-cond-2}).
The mathematical reason behind is that, 
according to the continuous mapping
theorem, $\|\sqrt{n}(f_{\mathbf{D}_{n},\Lambda_{n}}-\SVMplz)\|^2$
weakly converges to $\|\mathds{H}_P\|_H^2$. That is,
$\|f_{\mathbf{D}_{n},\Lambda_{n}}-\SVMplz\|^2$ converges with
rate $n$ while the confidence sets obtained from Theorem
\ref{theorem-asymptotic-confidence-sets}
are based on the rate $\sqrt{n}$.
By estimating quantiles of the distribution of 
$\|\mathds{H}_P\|_H^2$, it would be possible to derive
confidence bands for 
the whole function $\SVMplz$. However, estimating quantiles 
of the distribution of $\|\mathds{H}_P\|_H^2$ is a matter
of its own and cannot be done by use of the results of
Subsection \ref{subsec-theory} -- among other things 
because $\|\mathds{H}_P\|_H^2$ is not normally distributed
(as $\|\mathds{H}_P\|_H^2\geq0$).

\section{Simulations}

\subsection{Confidence sets for function values}\label{sec-simulation-1}

\textit{The model.} The situation
$$Y_i\;=\;f_{0}(X_i)+\varepsilon_i\;,\qquad i\in\{1,\dots,n\}
$$
is considered with the regression function
\begin{eqnarray}\label{simulation-1-true-function}
  f_0(x)\,=\,\log(x+2)+0.7\sin(3x)+0.7\cos(2x)\,.
\end{eqnarray}
The errors $\varepsilon_i$ are drawn i.i.d.\ from the 
standard normal distribution
and the covariates $X_i$ are drawn i.i.d.\ from the 
uniform distribution on $[0,5]$.
The simulation consists of 5000 data sets
with sample sizes $n$ equal to 250, 500, and 1000. 
The confidence sets apply to $\SVMplz$ with 
$\lambda_0=0.00001$ but 
the $L_1$-distance between $\SVMplz$ and 
the actual regression function $f_0$ is approximately equal
to 0.026 and the maximal pointwise distance is 
approximately equal to 0.091 so that the difference between
$\SVMplz$ and $f_0$ can be almost ignored for practical purposes here.
Three kinds of confidence sets are considered:
a univariate one for $\SVMplz(\tilde{x}_0)$
with $\tilde{x}_0=3$,
a multivariate one
for the four values $\SVMplz(\tilde{x})$, 
$\tilde{x}\in\{1,2,3,4\}$,
and a multivariate one
for the seven values $\SVMplz(\tilde{x})$, 
$\tilde{x}\in\{1,1.5,2,2.5,3,3.5,4\}$.
The nominal (asymptotic)
confidence level is 0.95.\\
\textit{Estimation.} The regularized kernel method
was applied with the Gaussian RBF kernel 
$k(x,x^\prime)=\exp(\gamma\|x-x^\prime\|_{\R^d}^2)$
and the logistic loss function with parameter $\sigma=0.5$.
Following \cite{caputo2002} and \cite[p.\ 9]{kernlab},
the hyperparameter
$\gamma$ of the kernel was fixed to 0.5 which is 
about the inverse of the median of the values 
$\|x_i-x_j^\prime\|_{\R^1}^2$.
The regularization parameter was chosen within the values
$$0.00001,\;\;0.00005,\;\;0.0001,\;\;0.0005,\;\;0.001,\;\;0.005,\;\;0.01
$$
in a data-driven way
by a fivefold cross-validation.\\ 
\textit{Performance results.}
Table \ref{table-example-1-results-univariate} 
lists the simulated coverage probabilities 
and, in case of the univariate
confidence interval, the average length($\pm$ standard deviation)
of the intervals obtained by 5000 data sets.
Figure \ref{figure-example-1-boxplots-variance-estimation} 
shows the boxplots for the estimates of 
the asymptotic variance $\Sigma_P$ of 
$\sqrt{n}(f_{\mathbf{D}_n,\Lambda_n}(\tilde{x_0})
          -\SVMplz(\tilde{x}_0)
         )
$
for the different
sample sizes. 
In addition, 
Figure \ref{figure-example-1-band-confidence-intervals}
shows the plot of the true function $\SVMplz$ 
and the pointwise univariate confidence 
interval for every $\tilde{x}\in[0,5]$
obtained for four different
data sets with $n=500$. This is only for 
illustration purposes and must not be mixed with a simultaneous
confidence band; the band around 
the true function is \emph{not} a simultaneous confidence band.

\begin{table}
\begin{center}
\begin{tabular}{r|cc|c|c}
      & \multicolumn{2}{c|}{1-dim.} 
      & 4-dim. & 7-dim. \\
  $n$ & \text{Cov.\,prob. }$(\%)$ & \text{ length} & 
        \text{Cov.\,prob. }$(\%)$ & 
        \text{Cov.\,prob. }$(\%)$  \\
  \hline
  250 & 92.7 & 0.61$\pm$0.09 & 91.6 & 79.4  \\
  500 & 94.0 & 0.44$\pm$0.04 & 93.1 & 91.1  \\
  1000 & 94.7 & 0.32$\pm$0.02 & 94.5 & 93.0  \\
\end{tabular}
\end{center}
\caption{Simulated coverage probability of the 
  confidence sets obtained by 5000 data sets in Subsection
  \ref{sec-simulation-1}.
}\label{table-example-1-results-univariate}
\end{table}
 
\begin{figure}
\begin{center}
  \includegraphics[width=0.7\textwidth]{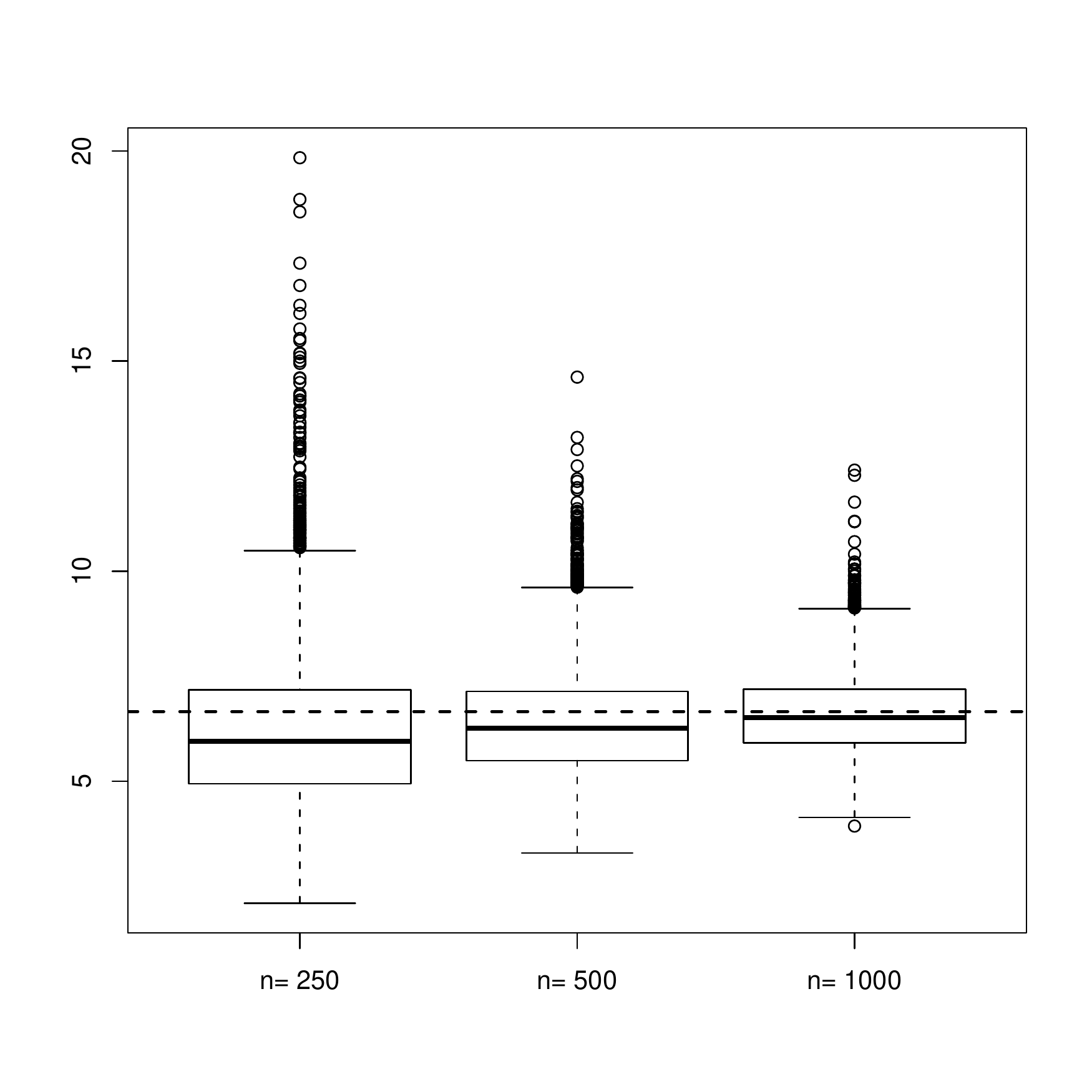}
  \caption{Boxplots for the estimation of 
   the asymptotic variance $\Sigma_P$ of 
   $\sqrt{n}(f_{\mathbf{D}_n,\Lambda_n}(\tilde{x_0})
             -\SVMplz(\tilde{x}_0)
            )
   $
   for the different
   sample sizes $n$ in Subsection \ref{sec-simulation-1}.}
  \label{figure-example-1-boxplots-variance-estimation}
\end{center}
\end{figure}

\begin{figure}
\begin{center}
  \includegraphics[width=\textwidth]{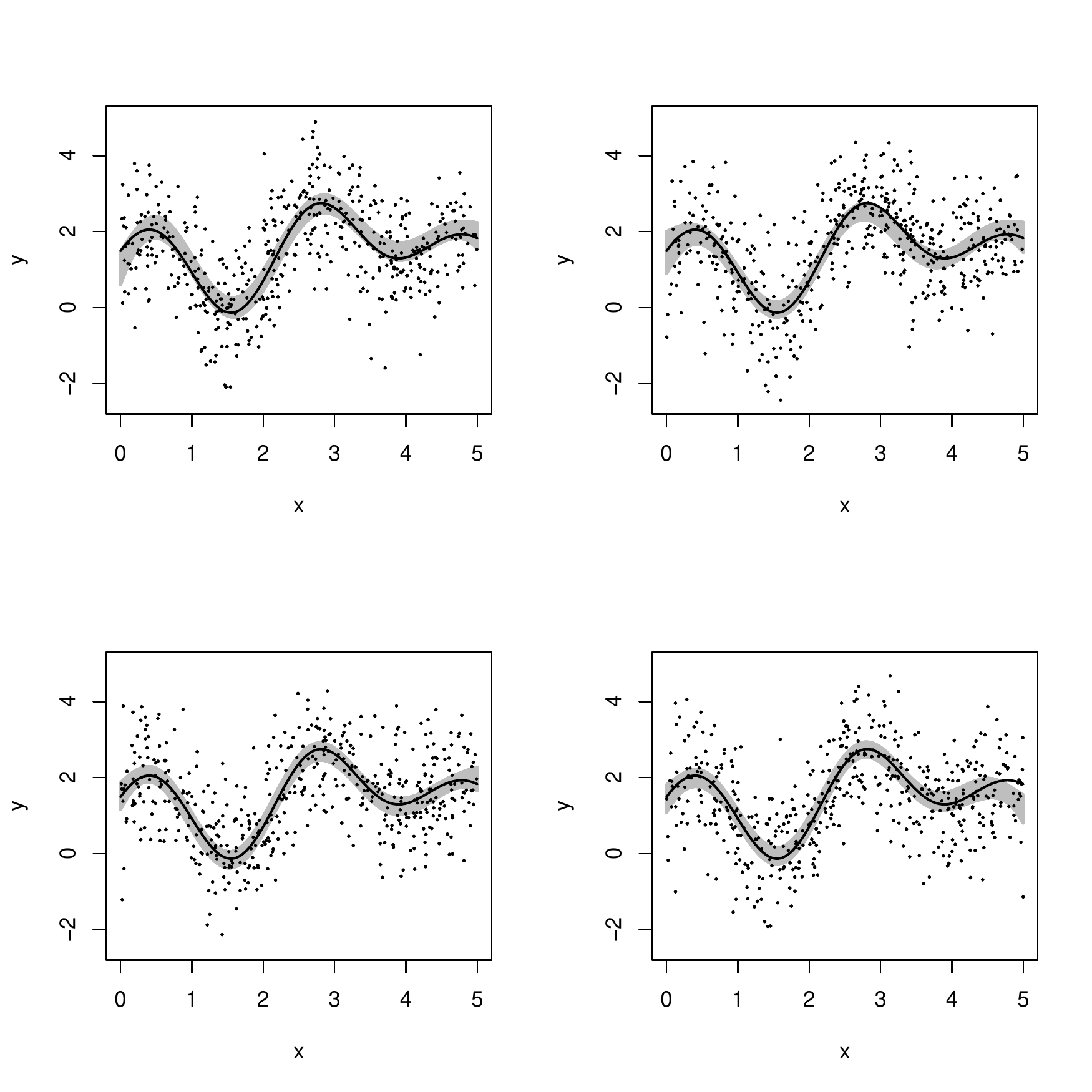}
  \caption{Estimated pointwise $0.95$-confidence intervals 
    (grey area) for four different data sets with sample size
    $n=500$ and the function $\SVMplz$ (solid line)
    in Subsection \ref{sec-simulation-1}.
    }\label{figure-example-1-band-confidence-intervals}
\end{center}
\end{figure}

\subsection{Confidence set for the gradient}\label{sec-simulation-2}

\textit{The model.} Two situations are considered, the
univariate one
$$Y_i\;=\;f_{0}(X_{i})
          +\varepsilon_i\;,\qquad i\in\{1,\dots,n\}\,,
$$
exactly as in Subsection \ref{sec-simulation-1},
and the multivariate one
$$Y_i\;=\;f_{0}(X_{i,1})+\sin(1.5X_{i,2})
          +\varepsilon_i\;,\qquad i\in\{1,\dots,n\}
$$
where $f_0$ is as in (\ref{simulation-1-true-function}).
The errors $\varepsilon_i$ are drawn i.i.d.\ from the 
standard normal distribution. 
In the univariate case,
the covariates $X_i$ 
are drawn i.i.d.\ from the 
uniform distribution on $[0,5]$ and,
in the multivariate case, the covariates
$X_{i,1}$ 
are also drawn i.i.d.\ from the 
uniform distribution on $[0,5]$ and
the covariates
$X_{i,2}$ 
are drawn i.i.d.\ from the 
uniform distribution on $[-1,1]$.
In both cases, we consider confidence sets for
$\psi(\SVMplz)=\partial \SVMplz(x_0)$ 
with $\lambda_0=0.00001$ where,
in the univariate case, $x_0=3$ and, in the
multivariate case, $x_0=(3,0)$.
Accordingly, the confidence set is an interval
in the univariate case and an ellipse in the multivariate case. 
The nominal (asymptotic)
confidence level is 0.95.\\
\textit{Estimation.} The regularized kernel method
was applied with the Gaussian RBF kernel 
and the logistic loss function with parameter $\sigma=0.5$.
Following \cite{caputo2002} and \cite[p.\ 9]{kernlab},
the hyperparameter
$\gamma$ of the kernel was fixed to $1/3$ which is 
about the inverse of the median of the values $\|x_i-x_j\|_{\R^2}^2$\,.
The regularization parameter was chosen as in 
Subsection \ref{sec-simulation-1}.\\ 
\textit{Performance results.}
Table \ref{table-example-2-results} 
lists the simulated coverage probabilities 
and, in case of the univariate
confidence interval, the average length($\pm$ standard deviation)
of the intervals obtained by 5000 data sets. For $n=1000$
in the multivariate case,
Figure \ref{figure-ellipses}
shows the estimates  
$\psi(f_{\mathbf{D}_n,\Lambda_n})=\partial f_{\mathbf{D}_n,\Lambda_n}(x_0)
$
obtained in the 5000 runs (gray points), the true value
$\psi(\SVMplz)$ (as cross $\times$), and the 
ellipse (dashed boundary) 
$$\Big\{w\in\R^m
    \Big|\,\,\big\|
             \Sigma_P^{-\frac{1}{2}}
             \big(w-\psi(\SVMplz)
             \big)
             \big\|^2_{\R^m}
             \leq \, \tfrac{\chi_{m,\alpha}^2}{n}
    \Big\}.
$$
in each plot. Asymptotically, this ellipse
contains the estimate
$\psi(f_{\mathbf{D}_n,\Lambda_n})$
with probability $0.95$.
In addition, each plot shows the estimate 
$\psi(f_{\mathbf{D}_n,\Lambda_n})$ (as black point)
and illustrates the estimated covariance matrix
$\hat{\Sigma}_n(\mathbf{D}_n,\Lambda_n)$
by showing the ellipse (solid boundary)
$$\Big\{w\in\R^m
    \Big|\,\,\big\|
             \hat{\Sigma}_n(\mathbf{D}_n,\Lambda_n)^{-\frac{1}{2}}
             \big(w-\psi(\SVMplz)
             \big)
             \big\|^2_{\R^m}
             \leq \, \tfrac{\chi_{m,\alpha}^2}{n}
    \Big\}
$$
given by the estimate $\hat{\Sigma}_n(\mathbf{D}_n,\Lambda_n)$
in one of the first four runs of the simulation.
\begin{figure}
\begin{center}
  \includegraphics[width=0.7\textwidth]{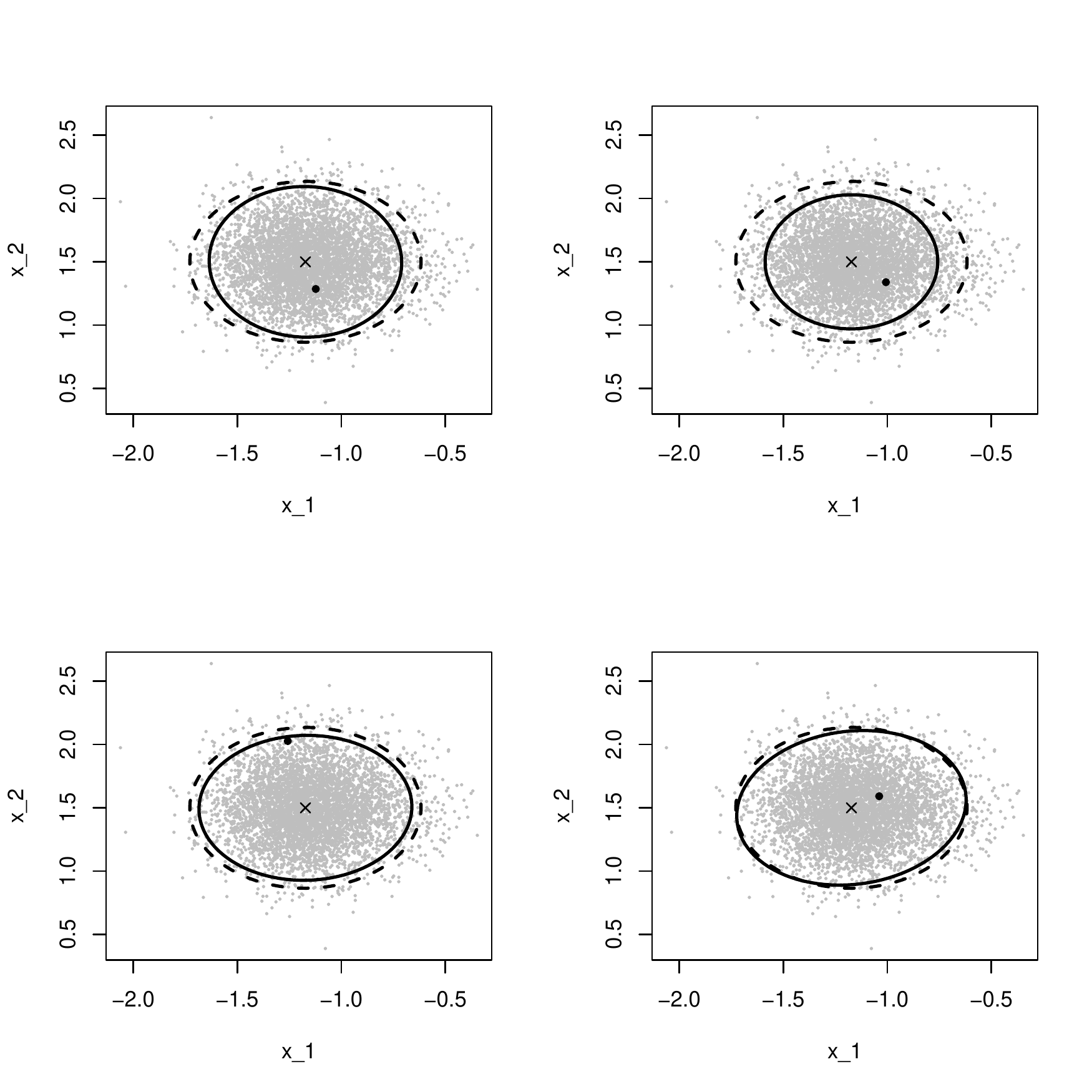}
  \caption{For $n=1000$ in the multivariate case,
    each plot shows the estimates  
    $\psi(f_{\mathbf{D}_n,\Lambda_n})$
    obtained in the 5000 runs (gray points), the true value
    $\psi(\SVMplz)$ (as cross $\times$), and the 
    ellipse (dashed boundary) which asymptotically contains the estimate
    $\psi(f_{\mathbf{D}_n,\Lambda_n})$
    with probability $0.95$. 
    Each of the four plots shows the estimate 
    $\psi(f_{\mathbf{D}_n,\Lambda_n})$ (plack point) and
    the ellipse where the true covariance 
    $\Sigma_P$ is replaced by the estimate
    $\hat{\Sigma}_n(\mathbf{D}_n,\Lambda_n)$
    (solid boundary) in one of the first four runs of
    the simulation in Subsection \ref{sec-simulation-2}.
    }\label{figure-ellipses}
\end{center}
\end{figure}

\begin{table}
\begin{center}
\begin{tabular}{r|cc|c|}
      & \multicolumn{2}{c|}{1-dim.} 
      & 2-dim.  \\
  $n$ & \text{Cov.\,prob. }$(\%)$ & \text{ length} & 
        \text{Cov.\,prob. }$(\%)$ \\
  \hline
  250 & 84.0 & 0.27$\pm$0.23 & 74.5   \\
  500 & 90.0 & 0.17$\pm$0.13 & 83.4   \\
  1000 & 91.5 & 0.11$\pm$0.09 & 91.3   \\
\end{tabular}
\end{center}
\caption{Simulated coverage probability of the 
  confidence sets obtained by 5000 data sets
  in Subsection \ref{sec-simulation-2}.
}\label{table-example-2-results}
\end{table}

\section{Conclusions}\label{sec-conclusions}

Regularized kernel methods 
constitute an important class of standard learning
algorithms in machine learning. 
As theoretical investigations 
concerning asymptotic properties
have manly focused on rates of convergence,
the lack of (asymptotic) results on statistical
inference is a serious limitation
for their use in mathematical statistics.
Therefore, the article derives asymptotically correct
confidence sets for $\psi(f_{P,\lambda_0})$
where $f_{P,\lambda_0}$ denotes the minimizer
of the regularized risk in the reproducing kernel
Hilbert space $H$ and $\psi:H\rightarrow\mathds{R}^m$
is any Hadamard-differentiable functional.
That is, the confidence sets do not apply to the 
minimizer $f^\ast$ of the unregularized risk, which would be
the quantity of primary interest, but to the minimizer of the
regularized risk. On the one hand, this is due to the 
so-called no-free-lunch theorem and obtaining confidence sets 
for $f^\ast$ would require a number of technical assumptions
which can hardly be made plausible in practical applications.
Without such assumptions, $f^\ast$ does not need to exist,
if it exists, it does not have to be unique, and the 
rate of convergence depends on unknown properties.
Technical assumptions can completely be avoided in this article;
all assumptions are simple and can easily be 
communicated to practitioners. On the other hand, it is exemplified
in a simulated example (Subsection \ref{sec-simulation-1}) 
that the difference between
$f^\ast$ and $f_{P,\lambda_0}$ is negligible for practical
purposes even for moderately small $\lambda_0>0$. 

The derivation of the confidence sets
is done by use of asymptotic normality of a large class of
regularized kernel methods and by the derivation
of a strongly consistent
estimator for the unknown covariance matrix of the 
limiting normal distribution. To this end, the following
non-trivial problems had to be solved satisfactorily:
(i) the derivation of a manageable formula for the
covariance matrix, which is accessible for a plug-in
estimator, (ii) strong consistency of
the plug-in estimator, (iii) the exclusion of
degeneracy of the covariance matrix by simple and week conditions,
and (iv) the derivation of an algorithm
for the calculation of the estimator 
which is computationally tractable also for moderately large sample sizes.

Applications include (multivariate) pointwise
confidence sets for values of $f_{P,\lambda_0}$
and confidence sets for gradients, integrals,
and norms. However, the derivation of simultaneous confidence bands
is a matter of further research. It follows from 
\cite[Theorem 3.1]{hable2012a}
that 
$\sqrt{n}
 \big\|f_{\mathbf{D}_n,\Lambda_n}-f_{P,\lambda_0}
 \big\|_\infty
 \,\leadsto\,\big\|\mathds{H}_P\big\|_\infty
$\,.
Hence, simultaneous 
confidence bands could be obtained if it is possible to derive
a consistent estimator for quantiles of 
$\big\|\mathds{H}_P\big\|_\infty$.

\section{Appendix: Proofs}\label{sec-appendix}

Assumption \ref{basic-assumptions} is valid
in the whole appendix.
Since the results of Section 
\ref{sec-asymptotic-confidence} are based on results and proofs in
\cite{hable2012a}, we have to recall the 
quite technical setting from \cite[\S\,A.1]{hable2012a} at first.

In order to shorten notation, define
$$L_{f}\;:\;\;\XY\;\rightarrow\;\R\,,\qquad
  (x,y)\;\mapsto\;L_{f}(x,y)\,=\,L\big(x,y,f(x)\big)
$$
for every function $f:\X\rightarrow\R$\,. Accordingly,
define 
$$\Ls_{f}(x,y)=\Ls\big(x,y,f(x)\big)\qquad\text{and}\qquad
  \Lss_{f}(x,y)=\Lss\big(x,y,f(x)\big)
$$ 
for every
$(x,y)\in\XY$.
As $L$ is a $P$-square-integrable Nemitski loss function
of order $p\in[1,\infty)$\,,
there is a $b\in L_{2}(P)$ and a constant $c\in(0,\infty)$
such that
\begin{eqnarray}\label{def-nemitski-loss}
  \big|L(x,y,t)\big|\;\leq\;b(x,y)+c|t|^{p}
  \qquad\forall\,(x,y,t)\in\XY\times\R\;.
\end{eqnarray}

Let
$$\mathcal{G}_{1}\;\;:=\;\;
  \big\{g:\XY\rightarrow\R\;
  \big|\;\;\exists\,z\in\R^{d+1}\;\;\text{such that}\;\;g=I_{(-\infty,z]}
  \big\}
$$
be the set of all indicator functions $I_{(-\infty,z]}$.
Define 
$\,c_{0}\,:=\,
  \sqrt{\frac{1}{\lambda_{0}}\int b\, \,dP\,}+1\,
$,
$$\mathcal{G}_{2}\;\;:=\;\;
  \left\{g:\XY\rightarrow\R\;
    \Bigg|\;\;
      \begin{array}{c}
        \exists\,f_{0}\in H\,, \;\;\exists\,f\in H 
           \;\;\text{such that} \\
        \|f_{0}\|_{H}\leq c_{0}\,,\,\;
        \|f\|_{H}\leq 1\;\,\text{and}\\
        g=\Ls_{f_{0}}f
      \end{array}
  \right\}\;,
$$
and
$$\G\;\;:=\;\;\G_{1}\cup\G_{2}\cup\{b\}\;.
$$
Let $\ell_{\infty}(\G)$ be the set of all bounded functions
$G:\,\G\rightarrow\R
$
with norm $\|G\|_{\infty}=\sup_{g\in\G}\big|G(g)\big|$\,.
Define
$$B_{S}\;:=\;
  \left\{G:\G\rightarrow\R\;
    \Bigg|\;\;
      \begin{array}{c}
        \exists\,\mu\not=0\;\text{a finite measure on}\;
                  \XY\;\text{such that} \\
        G(g)=\int g \,d\mu \;\,\forall\,g\in\G \,,\\
        b\in L_{2}(\mu)\,, \;\;
        b_{a}^{\prime}\in L_{2}(\mu)
        \;\,\,\forall\,a\in(0,\infty)
      \end{array}
  \right\}
$$  
and $B_{0}:=\textup{cl}\big(\textup{lin}(B_{S})\big)$ the closed linear
span of $B_{S}$
in $\ell_{\infty}(\G)$\,. That is, $B_{S}$ is a subset of
$\ell_{\infty}(\G)$ whose elements correspond to 
finite measures. The assumptions on $L$ and $P$ imply that
$\G\rightarrow\R,\;\;g\mapsto\int g\,dP$
is a well-defined element of $B_{S}$\,. Most often, we identify
an element $G\in B_S$ with its corresponding finite measure
$\mu$. That is, we write $\mu(g)=G(g)=\int g\,d\mu$ for
every $g\in\G$.

\smallskip

Let $\mu\in B_S$. 
Then,
$$S(\mu)\;:=\;f_{\mu,\lambda_{0}}\;=\;
  \text{arg}\inf_{f\in H}
    \bigg(\int L\big(x,y,f(x)\big)\,\mu\big(d(x,y)\big)
			\,+\,\lambda_{0}\|f\|_{H}^{2}
    \bigg)\;.
$$
This defines a map
$S:\,B_{S}\rightarrow H\,.
$
As the multiplication by a strictly positive real number does
not change the ``$\text{arg\,inf}$'', we have
\begin{eqnarray}\label{app-trace-back-to-standard-reg-parameter}
  f_{\mu,\lambda}\;=\;f_{\frac{\lambda_0}{\lambda}\mu,\lambda_0}
  \;=\;S\big(\tfrac{\lambda_0}{\lambda}\mu\big)
  \qquad\forall\,\mu\in B_S\,,
  \;\; \lambda\in(0,\infty)\;.
\end{eqnarray}

Let $\mu\in B_S$ such that $\mu(b)<P(b)+\lambda_0$.
Then, it is shown in \cite[Theorem A.8]{hable2012a} that,
$S$ is Hadamard differentiable in
$\mu$ tangentially to $B_0$. The derivative in $\mu$ is given by
\begin{eqnarray}\label{derivative-svm-functional-app}
  S^\prime_\mu(\nu)\,=
  -K_{\mu}^{-1}\Bigg(\!\int\!\! \Ls_{f_{\mu,\lambda_0}}(x,y)\Phi(x)\,
                     \nu(d(x,y))\!
             \Bigg)
  \qquad\forall\,\nu\in\textup{lin}(B_{S})\quad
\end{eqnarray}
and
\begin{eqnarray}\label{derivative-svm-functional-K-component-app}
  K_{\mu}:\;H\,\rightarrow\,H,\quad
  f\;\mapsto\;
  2\lambda_0 f+\!\int\!\!\Lss_{f_{\mu,\lambda_{0}}}(x,y)f(x)\Phi(x)
               \,\mu(d(x,y))\;.
\end{eqnarray}
Note that the integrals with respect to the 
finite signed measure $\nu$ in 
(\ref{derivative-svm-functional-app})
and the measure $\mu$ in
(\ref{derivative-svm-functional-K-component-app})
are Bochner integrals as the integrands are
$H$-valued functions. According to 
\cite[Lemma A.5]{hable2012a}, $K_\mu$ is
an invertible continuous linear operator and, according to 
\cite[Theorem A.8]{hable2012a},
the derivative
$S^\prime_\mu:B_0\rightarrow H$ is a continuous linear operator.
The following relation between $K_{\mu}$ and the random
$K_{\mathbf{D}_n,\Lambda_n}$ defined in
(\ref{def-of-random-K}) is valid:
\begin{eqnarray}\label{relation-random-K-deterministic-K}
  K_{\mathbf{D}_n,\Lambda_n}\;=\;
  \frac{\Lambda_n}{\lambda_0}
  K_{\frac{\lambda_0}{\Lambda_n}\mathds{P}_{\mathbf{D}_n}}\;.
\end{eqnarray}

If we identify the empirical measure
$\mathds{P}_{\mathbf{D}_{n}(\omega)}$
and $P$ with their corresponding elements in
$\ell_{\infty}(\G)$, it is shown in 
\cite[Lemma A.9]{hable2012a} that
\begin{eqnarray}\label{appendix-convergence-of-emp-process}
  \sqrt{n}\big(\mathds{P}_{\mathbf{D}_{n}}-P\big)
    \;\leadsto\;
    \mathds{G}_P
    \qquad\text{in}\quad\ell_{\infty}(\G)
\end{eqnarray}
where $\mathds{G}_P:\Omega\rightarrow\ell_{\infty}(\G)$
is a tight Borel-measurable Gaussian process.
Then, it is shown in 
\cite[Proof of Theorem 3.1]{hable2012a}
that
\begin{eqnarray}\label{appendix-asympt-normality-of-svm}
  \sqrt{n}
    \big(f_{\mathbf{D}_{n},\Lambda_{n}}-\SVMplz
    \big)
    \;\;\leadsto\;\;\mathds{H}_P\;=\;S^\prime_P\big(\mathds{G}_P\big)
    \qquad \text{in}\;\;H\,.
\end{eqnarray}

\begin{proof}
 \item[\textbf{Proof of Corollary 
       \ref{cor-finite-dim-asymptotic-normality}:%
       }
      ]
  According to the delta-method 
  \cite[Theorem 3.9.4]{vandervaartwellner1996},
  it follows from (\ref{appendix-asympt-normality-of-svm}) and 
  Hadamard-differentiability of $\psi$ in $\SVMplz$ that
  $$\sqrt{n}
    \Big(\psi\big(f_{\mathbf{D}_{n},\Lambda_{n}}
             \big)-
         \psi\big(\SVMplz\big)
    \Big)
    \;\;\leadsto\;\;\big\la\psi^\prime_{\SVMplz},\mathds{H}_P\big\ra\;.
  $$
  Since $f\mapsto\big\la\psi^\prime_{\SVMplz},f\big\ra$ is a continuous
  linear operator and $\mathds{H}$ is a zero-mean Gaussian
  process, it follows that the limit distribution
  is a multivariate normal distribution with mean zero, i.e.,
  the distribution of
  $\big\la\psi^\prime_{\SVMplz},\mathds{H}_P\big\ra$
  is equal to
  $\mathcal{N}_m(0,\Sigma_P)$
  for some covariate matrix $\Sigma_P\in\R^{m\times m}$;
  see e.g.\ \cite[\S\,3.9.2]{vandervaartwellner1996}.
\end{proof}

\begin{proof}
 \item[\textbf{Proof of Prop.\ 
       \ref{prop-covariance-matrix}:%
       }
      ]
  First, it is a direct consequence of the definition of
  the continuous linear operator $K_P$ that $K_P$ is self-adjoint and,
  accordingly, the inverse $K_P^{-1}$
  is again self-adjoint; see 
  \cite[Lemma VI.2.10]{dunford1958}.
  Define $f_j:=K_P^{-1}(\psi_{\SVMplz,j}^\prime)\,\in\,H$
  and note that \cite[(5.4)]{steinwart2008} implies
  \begin{eqnarray}\label{prop-covariance-matrix-p201}
    \Ls_{\SVMplz}\|f_j\|_H^{-1}f_j\;\in\;\G\,.
  \end{eqnarray}
  Since $K_P^{-1}$
  is self-adjoint, it follows
  for every $G\in\text{lin}(B_S)$ with corresponding signed
  measure $\mu$ that 
  \begin{eqnarray*}
    \lefteqn{
    \big\la \psi_{\SVMplz,j}^\prime, S_P^\prime(G)\big\ra
    \,\stackrel{(\ref{derivative-svm-functional-app})}{=}\,
         -\bigg\la f_j,
                  \int \!\!\Ls_{\SVMplz}\Phi\,d\mu
          \bigg\ra
    \,\stackrel{(\ast)}{=}\,
       -\!\int \!\!\Ls_{\SVMplz}\la f_j,\Phi\ra\,d\mu =       
    }\\
    &&=\,-\int \Ls_{\SVMplz}f_j\,d\mu
       \,\stackrel{(\ref{prop-covariance-matrix-p201})}=\,
           - \|f_j\|_H\cdot G\big(\Ls_{\SVMplz}\|f_j\|_H^{-1}f_j\big)\;.
            \qquad\qquad\qquad\quad\,\,
  \end{eqnarray*}
  where $(\ast)$ follows from interchangeability of Bochner integrals
  with continuous linear operators;
  see e.g.\ \cite[Theorem 3.10.16 and Remark 3.10.17]{denkowski2003}.
  Next, it follows from continuity of $S_P^\prime$ that
  \begin{eqnarray}\label{prop-covariance-matrix-p204}
    \big\la \psi_{\SVMplz,j}^\prime, S_P^\prime(G)\big\ra
    \;=\;-\|f_j\|_H\cdot G\big(\Ls_{\SVMplz}\|f_j\|_H^{-1}f_j\big)
  \end{eqnarray}
  is valid even for every $G\in B_0$ where $B_0$ denotes
  the closed linear span of $B_S$ in $\ell_\infty(\G)$.
  Since $\mathds{G}_P$ takes its values in $B_0$, it follows from 
  $\mathds{H_P}=S_P^\prime(\mathds{G}_P)$ now
  that
  \begin{eqnarray}\label{prop-covariance-matrix-p202}
    \big\la \psi_{\SVMplz,j}^\prime, \mathds{H}_P\big\ra
    \;=\;-\|f_j\|_H\cdot 
         \mathds{G}_P\big(\Ls_{\SVMplz}\|f_j\|_H^{-1}f_j\big)
    \quad\forall\,j\in\{1,\dots,m\}\,.\;\,
  \end{eqnarray}
  According to (\ref{appendix-convergence-of-emp-process})
  and (\ref{prop-covariance-matrix-p201}),
  $$\left(
      \begin{array}{c}
        \mathds{G}_P\big(\Ls_{\SVMplz}\|f_1\|_H^{-1}f_1\big)\\
        \vdots \\
        \mathds{G}_P\big(\Ls_{\SVMplz}\|f_m\|_H^{-1}f_m\big)\\
      \end{array}
    \right)
    \;\sim\;\mathcal{N}_m\big(0,\tilde{\Sigma}_P\big)
  $$
  where $\tilde{\Sigma}_P$ is the covariance matrix
  of 
  $$\Big(\Ls_{\SVMplz}(X,Y)\|f_1\|_H^{-1}f_1(X),\dots,
        \Ls_{\SVMplz}(X,Y)\|f_m\|_H^{-1}f_m(X)
    \Big)^{\scriptscriptstyle\mathsf{T}}\;;
  $$
  see, e.g., \cite[p.\ 81f]{vandervaartwellner1996}.
  Let $C$ denote the diagonal matrix with diagonal entries
  $\|f_1\|_H,\dots,\|f_m\|_H$. 
  Then, it follows from 
  (\ref{prop-covariance-matrix-p202}) that
  \begin{eqnarray}\label{prop-covariance-matrix-p203}
    \Sigma_P\;=\;
    \text{Cov}\Big(\big\la \psi_{\SVMplz}^\prime, \mathds{H}_P\big\ra\Big)
    \;=\;C\tilde{\Sigma}_P C\,.
  \end{eqnarray}
  Since, according to the reproducing property
  and self-adjointness of $K_P^{-1}$,
  \begin{eqnarray*}
    f_j(x)
    &=&\Big(K_P^{-1}(\psi_{\SVMplz,j}^\prime)\Big)(x)
       \;=\;\Big\la K_P^{-1}\big(\psi_{\SVMplz,j}^\prime\big),\Phi(x)
            \Big\ra\;=\\
    &=&\Big\la \psi_{\SVMplz,j}^\prime,K_P^{-1}\big(\Phi(x)\big)
       \Big\ra
       \qquad\forall\,j\in\{1,\dots,m\}\,,
  \end{eqnarray*}
  it follows that 
  $-\Ls_{\SVMplz}\|f_j\|_H^{-1}f_j=\|f_j\|_H^{-1}g_{P,\lambda_0,j}$
  where $g_{P,\lambda_0,j}$ denotes the $j$-th component
  of $g_{P,\lambda_0}$.  
  Hence, $\tilde{\Sigma}_P=\text{Cov}\big(C^{-1}g_{P,\lambda_0}(X,Y)\big)$
  so that (\ref{prop-covariance-matrix-p203}) implies
  $\Sigma_P=\text{Cov}\big(g_{P,\lambda_0}(X,Y)\big)$.
\end{proof}

\begin{lemma}\label{lemma-measurability-of-covariance-estimator}
  Under the Assumptions of Theorem 
  \ref{theorem-consistency-covariance-estimator},
  the covariance estimator $\hat{\Sigma}_n(\mathbf{D}_n,\Lambda_n)$
  is measurable with respect to
  $\mathcal{A}$ and $\mathds{B}^{\otimes m^2}$.
\end{lemma}
\begin{proof}
 \item[\textbf{Proof of Lemma 
       \ref{lemma-measurability-of-covariance-estimator}:%
       }
      ]
  It has to be shown
  that
  $g_{\mathbf{D}_{n},\Lambda_{n}}(X_i,Y_i)$
  is measurable for every $i\in\{1,\dots,n\}$.\\
  First, note that
  $\omega\mapsto
   f_{\mathbf{D}_{n}(\omega),\Lambda_{n}(\omega)}
  $
  is measurable because:
  for every fixed $\lambda>0$, the map
  $D\mapsto f_{D,\lambda}$ is continuous on $(\XY)^n$
  according to \cite[Lemma 5.13]{steinwart2008}
  and, for every fixed $D\in(\XY)^n$, the map
  $\lambda\mapsto f_{D,\lambda}$ is continuous 
  on $(0,\infty)$ according to 
  \cite[Theorem 5.17]{steinwart2008}; 
  hence,
  $(D,\lambda)\mapsto f_{D,\lambda}$ is a Caratheodory function and,
  therefore, measurable, see e.g.\ 
  \cite[Theorem 2.5.22]{denkowski2003}.\\
  Secondly, we show measurability of 
  $K_{\mathbf{D}_{n},\Lambda_{n}}^{-1}\big(\Phi(X_i)\big)$.
  To this end, define
  $$A_{D,\lambda,g}:\;H\,\rightarrow\,H,\quad
    f\;\mapsto\;
    \frac{1}{n}
    \sum_{i=1}^{n}
      \Lss\big(x_i,y_i,g(x_i)
          \big)f(x_i)\Phi(x_i)
  $$
  and
  $$K_{D,\lambda,g}:\;H\,\rightarrow\,H,\quad
    f\;\mapsto\;
    2\lambda f+A_{D,\lambda,g}(f)
  $$
  for every $D=\big((x_1,y_1),\dots,(x_n,y_n)\big)\in(\XY)^n$,
  $\lambda\in(0,\infty)$, and $g\in H$.
  That is, 
  $K_{\mathbf{D}_{n},\Lambda_{n}}=
   K_{\mathbf{D}_{n},\Lambda_{n},
      f_{\mathbf{D}_{n},\Lambda_{n}}%
     }
  $.
  The assumptions imply that,
  \begin{eqnarray}\label{lemma-measurability-of-covariance-estimator-p1}
    (\X\!\times\!\Y)^n\!\times\!(0,\infty)\!\times\! H
    \,\rightarrow\,H\,,\quad
    (D,\lambda,g)\,\mapsto\,K_{D,\lambda,g}(f)
    \quad\text{is continuous}\;\;
  \end{eqnarray}
  for every $f\in H$. Note that
  \begin{eqnarray*}
    \la f, A_{D,\lambda,g}(f)\ra
    &=&
       \frac{1}{n}
       \sum_{i=1}^{n}
         \Lss\big(x_i,y_i,g(x_i)
             \big)f(x_i)\la f,\Phi(x_i)\ra
       \;=\\
    &=&\frac{1}{n}\sum_{i=1}^{n}
      \Lss\big(x_i,y_i,g(x_i)
          \big)\big(f(x_i)\big)^2
       \;\geq\;0
  \end{eqnarray*}  
  because convexity of $t\mapsto L(x,y,t)$ implies
  $\Lss(x,y,t)\geq 0$. Hence,
  $$\|K_{D,\lambda,g}(f)\|_H^2\;=\;
    4\lambda^2\|f\|^2+2\lambda\la f, A_{D,\lambda,g}(f)\ra
    +\|A_{D,\lambda,g}(f)\|_H^2
    \;\geq\;4\lambda^2\|f\|^2
  $$
  for every $f\in H$, and this implies
  \begin{eqnarray}\label{lemma-measurability-of-covariance-estimator-p2}
    \|K_{D,\lambda,g}^{-1}\|\;\leq\;\frac{1}{2\lambda}
    \qquad\forall\,(D,\lambda,g)\in(\XY)^n\times(0,\infty)\times H\,.
  \end{eqnarray}
  Let the sequence $(D_\ell,\lambda_\ell,g_\ell)$, $\ell\in\N$,
  converge to some $(D,\lambda,g)\in(\XY)^n\times(0,\infty)\times H$.
  Fix any $f\in H$ and denote $h:=K_{D,\lambda,g}^{-1}(f)$.
  Then,
  \begin{eqnarray*}
    \lefteqn{
    \big\|K_{D_\ell,\lambda_\ell,g_\ell}^{-1}(f)
          -K_{D,\lambda,g}^{-1}(f)
    \big\|_H
    \;=\;\big\|K_{D_\ell,\lambda_\ell,g_\ell}^{-1}(K_{D,\lambda,g}(h))-h
         \big\|_H
    \;=
    }\\
    &=&\big\|K_{D_\ell,\lambda_\ell,g_\ell}^{-1}
                 \big(K_{D,\lambda,g}(h)
                      -K_{D_\ell,\lambda_\ell,g_\ell}(h)
                 \big)
        \big \|_H\;\leq\\
    &\stackrel{(\ref{lemma-measurability-of-covariance-estimator-p2})
               }{\leq}&
         \frac{1}{2\lambda_\ell}
         \big\|K_{D,\lambda,g}(h)-K_{D_\ell,\lambda_\ell,g_\ell}(h)
         \big\|_H
         \;\xrightarrow[\;\ell\rightarrow\infty\;]{}\;0
         \qquad\qquad\qquad\qquad
  \end{eqnarray*}
  according to (\ref{lemma-measurability-of-covariance-estimator-p1}).
  That is, $(D,\lambda,g)\mapsto K_{D,\lambda,g}^{-1}(f)$
  is continuous for every fixed $f\in H$. Since
  $f\mapsto K_{D,\lambda,g}^{-1}(f)$ is continuous for 
  every fixed $(D,\lambda,g)$, the function
  $\big((D,\lambda,g),f\big)\mapsto K_{D,\lambda,g}^{-1}(f)$
  is a Caratheodory function and, therefore, measurable. Since
  $f_{\mathbf{D}_{n},\Lambda_{n}}$ is measurable 
  as shown above
  and
  $K_{\mathbf{D}_{n},\Lambda_{n}}=
   K_{\mathbf{D}_{n},\Lambda_{n},
      f_{\mathbf{D}_{n},\Lambda_{n}}%
     }
  $, it follows that
  $K_{\mathbf{D}_{n},\Lambda_{n}}^{-1}\big(\Phi(X_i)\big)$ 
  is measurable.\\
  Finally, measurability of the estimator 
  $\psi^\prime_{\mathbf{D}_{n},\Lambda_{n}}$,
  measurability of $f_{\mathbf{D}_{n},\Lambda_{n}}$, and
  measurability of 
  $K_{\mathbf{D}_{n},\Lambda_{n}}^{-1}\big(\Phi(X_i)\big)$ 
  imply measurability of
  $$g_{\mathbf{D}_{n},\Lambda_{n}}(X_i,Y_i)\;=\;
    -\Ls\big(X_i,Y_i,f_{\mathbf{D}_{n},\Lambda_{n}}(X_i)
        \big)
     \big\la\psi^\prime_{\mathbf{D}_{n},\Lambda_{n}},
             K_{\mathbf{D}_{n},\Lambda_{n}}^{-1}\big(\Phi(X_i)\big)
     \big\ra\;.
  $$
\end{proof}

\begin{proof}
 \item[\textbf{Proof of Theorem 
       \ref{theorem-consistency-covariance-estimator}:%
       }
      ]
  For every $j\in\{1,\dots,m\}$, let
  $\psi^\prime_{\SVMplz,j}\in H$ denote the $j$-th component
  of $\psi^\prime_{\SVMplz}$ and, accordingly, let
  $\psi^\prime_{\mathbf{D}_n,\Lambda_n,j}\in H$,
  $g_{\mathbf{D}_{n},\Lambda_{n},j}$, and
  $g_{P,\lambda_{0},j}$ denote the $j$-th component
  of $\psi^\prime_{\mathbf{D}_n,\Lambda_n}$,
  $g_{\mathbf{D}_{n},\Lambda_{n}}$, and
  $g_{P,\lambda_{0}}$ respectively. 
  Define
  $Z_i=(X_i,Y_i)$ for every $i\in\N$. 
  Measurability of $g_{\mathbf{D}_{n},\Lambda_{n},j}(Z_i)$
  is shown in the proof of 
  Lemma \ref{lemma-measurability-of-covariance-estimator}.
  Define $a:=\|\SVMplz\|_\infty+1\,\in\,[1,\infty)$
  and 
  $c:=\max_j
      \big\|\psi^\prime_{f_{P,\lambda_{0},j}}\big\|_H\cdot
      \big\|K_P^{-1}\big\|\cdot\|k\|_\infty
  $
  where $\big\|K_P^{-1}\big\|$ denotes the operator norm 
  of the continuous linear operator $K_P^{-1}$.
  Then, the definition of $g_{P,\lambda_0,j}$
  and (\ref{theorem-sqrt-n-consistency-1}) imply
  \begin{eqnarray}\label{theorem-consistency-covariance-estimator-p9}
    \big|g_{P,\lambda_0,j}(z)\big|\;\leq\;
    c\cdot b_a^\prime(z)
    \qquad\forall\,z\in\XY\;.
  \end{eqnarray}
  Hence, $g_{P,\lambda_0,j}$ is $P$-square integrable.
  Fix any $j,\ell\in\{1,\dots,m\}$.
  We have to show
  \begin{eqnarray}
    \label{theorem-consistency-covariance-estimator-p1}
      \frac{1}{n}\sum_{i=1}^n 
        g_{\mathbf{D}_{n},\Lambda_{n},j}(Z_i)
      &\xrightarrow[\;n\rightarrow\infty\;]{\,\text{a.s.}\,}&
      \mathbb{E}\big[g_{P,\lambda_{0},j}(Z_1)\big]  \\
    \label{theorem-consistency-covariance-estimator-p2}
      \frac{1}{n}\!\sum_{i=1}^n 
        g_{\mathbf{D}_{n},\Lambda_{n},j}(Z_i)
        g_{\mathbf{D}_{n},\Lambda_{n},\ell}(Z_i)
      &\xrightarrow[\;n\rightarrow\infty\;]{\,\text{a.s.}\,}&
      \mathbb{E}\big[g_{P,\lambda_{0},j}(Z_1)
                     g_{P,\lambda_{0},\ell}(Z_1)
                \big]\,.\qquad\quad
  \end{eqnarray}
  According to \cite[Lemma A.9]{hable2012a}, 
  $\G$ is a $P$-Donsker class and, therefore,
  a $P$-Glivenko-Cantelli class almost sure;
  see \cite[p.\ 82]{vandervaartwellner1996}.
  Hence,
  $\sup_{g\in\G}\big|\mathds{P}_{\mathbf{D}_{n}}(g)-P(g)\big|
   \,\longrightarrow\,0
  $ almost surely and, therefore, there is a measurable 
  set $\Omega_0\in\mathcal{A}$ such that $Q(\Omega_0)=1$ and
  $\sup_{g\in\G}\big|\mathds{P}_{\mathbf{D}_{n}(\omega)}(g)-P(g)\big|
   \,\longrightarrow\,0
  $
  for every $\omega\in\Omega_0$; see
  \cite[\S\,1.9 and Lemma 1.2.3]{vandervaartwellner1996}.
  Due to the law
  of large numbers, we can choose $\Omega_0\in\mathcal{A}$ in such 
  a way that,
  for every $\omega\in\Omega_0$,
  in addition,
  $\frac{1}{n}\sum_{i=1}^n g_{P,\lambda_{0},j}(Z_i(\omega))
  $
  and
  $\frac{1}{n}\sum_{i=1}^n g_{P,\lambda_{0},j}(Z_i(\omega))
                     g_{P,\lambda_{0},\ell}(Z_i(\omega))
  $ 
  and
  $\frac{1}{n}\sum_{i=1}^n b_a^\prime(Z_i(\omega))
  $ 
  and 
  $\frac{1}{n}\sum_{i=1}^n b_a^\prime(Z_i(\omega))^2
  $ converge to their expectations
  for $n\rightarrow\infty$. Furthermore, 
  due to the assumptions on 
  $\psi^\prime_{\mathbf{D}_n,\Lambda_n}$ and
  $\Lambda_n$,
  the set $\Omega_0$ can also be chosen
  in such a way that, in addition, 
  $\|\psi^\prime_{\mathbf{D}_n(\omega),\Lambda_n(\omega)}-
     \psi^\prime_{\SVMplz}
   \|_H
   \longrightarrow0
  $
  and
  $\Lambda_n(\omega)\longrightarrow\lambda_0$
  for every $\omega\in\Omega_0$.
  Fix any $\omega\in\Omega_0$;
  define $D_n:=\mathbf{D}_n(\omega)$ and $\lambda_n:=\Lambda_n(\omega)$
  for every $n\in\N$ and
  $(x_i,y_i):=z_i:=Z_i(\omega)$ for every $i\in\N$. That is, we have
  \begin{eqnarray} 
    &&\hspace*{-1.5cm}\label{theorem-consistency-covariance-estimator-p3}
      \lim_{n\rightarrow\infty}
    \sup_{g\in\G}\big|\mathds{P}_{D_{n}}(g)-P(g)\big|
    \;=\;0\,,
    \qquad \lim_{n\rightarrow\infty}\lambda_n\;=\;\lambda_0\,,\\
     &&\hspace*{-1.5cm}
              \label{theorem-consistency-covariance-estimator-p501}
      \lim_{n\rightarrow\infty}
      \big\|\psi^\prime_{D_n,\lambda_n}-
            \psi^\prime_{\SVMplz}
      \big\|_H
      \;=\;0\,,\\
     &&\hspace*{-1.5cm}
              \label{theorem-consistency-covariance-estimator-p101}
      \lim_{n\rightarrow\infty}
      \frac{1}{n}\sum_{i=1}^n 
        g_{P,\lambda_{0},j}(z_i)
      \;=\;\mathbb{E}_P\big[g_{P,\lambda_{0},j}
                       \big]\,,\\
     &&\hspace*{-1.5cm}\label{theorem-consistency-covariance-estimator-p4}
      \lim_{n\rightarrow\infty}
      \frac{1}{n}\sum_{i=1}^n 
        g_{P,\lambda_{0},j}(z_i)g_{P,\lambda_{0},\ell}(z_i)
      \;=\;\mathbb{E}_P\big[g_{P,\lambda_{0},j}g_{P,\lambda_{0},\ell}
                       \big]\,,\\
   &&\hspace*{-1.5cm}\label{theorem-consistency-covariance-estimator-p5}
      \lim_{n\rightarrow\infty}
      \frac{1}{n}\sum_{i=1}^n b_a^\prime(z_i)
      \;=\;\mathbb{E}_P b_a^\prime\,,\quad\text{and}\quad
    \lim_{n\rightarrow\infty}
      \frac{1}{n}\sum_{i=1}^n b_a^\prime(z_i)^2\;=\;
      \mathbb{E}_P b_a^{\prime\, 2}\,.
  \end{eqnarray}
  It is shown in \cite[(46) and (47)]{hable2012a} that 
  $S:B_S\rightarrow H,\,\,\mu\mapsto S(\mu)=f_{\mu,\lambda_0}$
  is continuous in $P$ and, therefore,
  \begin{eqnarray}\label{theorem-consistency-covariance-estimator-p6}
    \lim_{n\rightarrow\infty}f_{D_n,\lambda_n}
    \;\stackrel{(\ref{app-trace-back-to-standard-reg-parameter})}{=}\;
    \lim_{n\rightarrow\infty}
      S\big(\tfrac{\lambda_0}{\lambda_n}\mathds{P}_{D_n}\big)
    \;\stackrel{(\ref{theorem-consistency-covariance-estimator-p3})}{=}\;
    S(P)\;=\;f_{P,\lambda_0}\;.
  \end{eqnarray}
  In view of (\ref{theorem-consistency-covariance-estimator-p101}),
  it suffices to prove
  \begin{eqnarray}\label{theorem-consistency-covariance-estimator-p102}
      \frac{1}{n}\!\sum_{i=1}^n 
        g_{D_{n},\lambda_{n},j}(z_i)
      -\frac{1}{n}\!\sum_{i=1}^n 
        g_{P,\lambda_{0},j}(z_i)
      \;\xrightarrow[\;n\rightarrow\infty\;]{}\;
      0\,\;\;
  \end{eqnarray}
  in order to prove 
  (\ref{theorem-consistency-covariance-estimator-p1}). \\
  First, it is shown in the following that, for the fixed sequence
  $(z_n)_{n\in\N}\in\XY$, there is 
  an $n_j\in\N$ and a
  sequence 
  $(\varepsilon_{j,n})_{n\in\N}\subset[0,\infty)$ such that
  $\lim_{n\rightarrow\infty}\varepsilon_{j,n}=0$ and,
  for every $n\geq n_j$ and for every $z=(x,y)\in\XY$,
  \begin{eqnarray}\label{theorem-consistency-covariance-estimator-p11}
    \big|g_{D_{n},\lambda_{n},j}(z)-g_{P,\lambda_{0},j}(z)
    \big|
    \;\leq\;\varepsilon_{j,n}+\varepsilon_{j,n}\cdot b_a^\prime(z)\;.
  \end{eqnarray}
  To this end, note that 
  it is shown in \cite[(43)]{hable2012a}
  that $\mu\mapsto K_\mu^{-1}$ is continuous
  in $P$ and, therefore, it follows from 
  (\ref{theorem-consistency-covariance-estimator-p3}) that
  \begin{eqnarray}\label{theorem-consistency-covariance-estimator-p12}
    K_{D_{n},\lambda_{n}}^{-1}
    \;\stackrel{(\ref{relation-random-K-deterministic-K})}{=}\;
    \frac{\lambda_0}{\lambda_n}
    K_{\frac{\lambda_0}{\lambda_n}\mathds{P}_{D_n}}^{-1}
    \;\xrightarrow[\;n\rightarrow\infty\;]{}\;
    K_P^{-1}
    \qquad\text{in operator norm}.
  \end{eqnarray}
  The definitions imply
  \begin{eqnarray}\label{theorem-consistency-covariance-estimator-p13}
    \lefteqn{
    \big|g_{D_{n},\lambda_{n},j}(z)-g_{P,\lambda_{0},j}(z)
    \big|
    \;\leq\;
    }\\
    &\!\leq&\!\!\!
     \big|\big\la\psi^\prime_{D_{n},\lambda_{n},j},
                 K_{D_{n},\lambda_{n}}^{-1}\big(\Phi(x)\big)
          \big\ra
     \big|
     \cdot\big|\Ls_{f_{D_{n},\lambda_{n}}}(z)-
               \Ls_{\SVMplz}(z)
          \big|
     \;+\nonumber\\
    &&\!\!+\,
       \Big|
         \big\la\psi^\prime_{D_{n},\lambda_{n},j},
                K_{D_{n},\lambda_{n}}^{-1}\big(\Phi(x)\big)
         \big\ra
         -
         \big\la\psi^\prime_{\SVMplz,j},
                K_{P}^{-1}\big(\Phi(x)\big)
         \big\ra
       \Big|
       \!\cdot\!\big|\Ls_{\SVMplz}(z)\big|.\nonumber
  \end{eqnarray}
  Due to (\ref{theorem-consistency-covariance-estimator-p6}), 
  there is an $n_j\in\N$ such that
  $\|f_{D_{n},\lambda_{n}}\|_\infty<\|\SVMplz\|_\infty+1=a$
  for every $n\geq n_j$. Hence, the first summand converges to
  0 uniformly in $z\in\XY$
  because of (\ref{theorem-consistency-covariance-estimator-p6}),
  $\|\Phi(\tilde{x})\|_H\leq\|k\|_\infty\;\forall\,\tilde{x}\in\X$,
  (\ref{theorem-consistency-covariance-estimator-p501}), 
  (\ref{theorem-consistency-covariance-estimator-p12}), and
  \begin{eqnarray*}
    \lefteqn{
    \big|\big\la\psi^\prime_{D_{n},\lambda_{n},j},
                 K_{D_{n},\lambda_{n}}^{-1}\big(\Phi(x)\big)
         \big\ra
    \big|
     \cdot\big|\Ls_{f_{D_{n},\lambda_{n}}}(z)-
               \Ls_{\SVMplz}(z)
          \big|
    \;\leq}\\
    &&\stackrel{(\ref{theorem-sqrt-n-consistency-1})}{\leq}\;
    \big\|\psi^\prime_{D_{n},\lambda_{n},j}\big\|_H
    \cdot \big\|K_{D_{n},\lambda_{n}}^{-1}\big\|
    \cdot\sup_{\tilde{x}\in\X}\|\Phi(\tilde{x})\|_H\cdot
    b_a^{\prime\prime}\cdot
    \big\|f_{D_{n},\lambda_{n}}-\SVMplz\big\|_\infty\;. 
  \end{eqnarray*}
  For $\psi^\prime_{\ast,j}\in H$,
  let $\psi^\prime_{\ast,j}\circ K_{\mu}^{-1}$ denote the
  continuous linear operator 
  $h\mapsto\la\psi^\prime_{\ast,j}, K_{\mu}^{-1}(h)\ra$. Then, 
  the
  second summand in 
  (\ref{theorem-consistency-covariance-estimator-p13})
  is bounded via
  \begin{eqnarray*}
    \lefteqn{
    \Big|
         \big\la\psi^\prime_{D_{n},\lambda_{n},j},
                K_{D_{n},\lambda_{n}}^{-1}\big(\Phi(x)\big)
         \big\ra
         -
         \big\la\psi^\prime_{\SVMplz,j},
                K_{P}^{-1}\big(\Phi(x)\big)
         \big\ra
    \Big|
    \!\cdot\!\big|\Ls_{\SVMplz}(z)\big|
    \;\leq}\\
    &&\stackrel{(\ref{theorem-sqrt-n-consistency-1})}{\leq}\;
      \big\|\psi^\prime_{D_{n},\lambda_{n},j}\circ 
                K_{D_{n},\lambda_{n}}^{-1}
            -\psi^\prime_{\SVMplz,j}\circ K_{P}^{-1}
      \big\|
      \cdot\sup_{\tilde{x}\in\X}\|\Phi(\tilde{x})\|_H\cdot
       b_a^{\prime}(z) \;,\qquad  
  \end{eqnarray*}
  and 
  $\big\|\psi^\prime_{D_{n},\lambda_{n},j}\circ 
                K_{D_{n},\lambda_{n}}^{-1}
            -\psi^\prime_{\SVMplz,j}\circ K_{P}^{-1}
      \big\|
      \cdot\sup_{\tilde{x}\in\X}\|\Phi(\tilde{x})\|_H
  $
  converges to zero because of  
  $\|\Phi(\tilde{x})\|_H\leq\|k\|_\infty\;\forall\,\tilde{x}\in\X$,
  (\ref{theorem-consistency-covariance-estimator-p501}), and
  (\ref{theorem-consistency-covariance-estimator-p12}). 
  This proves that we can choose a null-sequence 
  $(\varepsilon_{j,n})_{n\in\N}\subset[0,\infty)$
  such that (\ref{theorem-consistency-covariance-estimator-p11}) 
  is fulfilled for every $n\geq n_j$ and every $z\in\XY$.
  Accordingly, there is 
  an $n_\ell\in\N$ and a
  sequence 
  $(\varepsilon_{\ell,n})_{n\in\N}\subset[0,\infty)$ such that
  $\lim_{n\rightarrow\infty}\varepsilon_{\ell,n}=0$ and,
  for every $n\geq n_\ell$ and for every $z\in\XY$,
  assertion 
  (\ref{theorem-consistency-covariance-estimator-p11})
  with $j$ replaced by $\ell$ is fulfilled. Then, due to
  (\ref{theorem-consistency-covariance-estimator-p9}),
  \begin{eqnarray}\label{theorem-consistency-covariance-estimator-p14}
    \big|g_{D_{n},\lambda_{n},\ell}(z)
    \big|
    \;\leq\;
    \varepsilon_{\ell,n}+(c+\varepsilon_{\ell,n})\cdot b_a^\prime(z)
    \qquad\forall\,z\in\XY\,,\;\;n\geq n_\ell\;.
  \end{eqnarray}
  Define $\varepsilon_n:=\max\{\varepsilon_{j,n},\varepsilon_{\ell,n}\}$
  for every $n\in\N$. Then, for every $n\geq n_j$,
  \begin{eqnarray*}
    \lefteqn{
    \bigg|
      \frac{1}{n}\!\sum_{i=1}^n 
        g_{D_{n},\lambda_{n},j}(z_i)
      -\frac{1}{n}\!\sum_{i=1}^n 
        g_{P,\lambda_{0},j}(z_i)
    \bigg|
    \;\leq\;
       \frac{1}{n}\sum_{i=1}^n
       \big|g_{D_{n},\lambda_{n},j}(z_i)-
            g_{P,\lambda_{0},j}(z_i)
       \big|
    }\\
    &&
     \stackrel{(\ref{theorem-consistency-covariance-estimator-p11})
              }{\leq}\;
          \frac{1}{n}\sum_{i=1}^n
            \big(\varepsilon_{n}+\varepsilon_{n}\cdot b_a^\prime(z_i)
            \big)
            \;\xrightarrow[\;n\rightarrow\infty\;]{}\;0
     \qquad\qquad\qquad\qquad\qquad\qquad\quad
  \end{eqnarray*}
  where convergence to 0 follows from
  $\lim_{n\rightarrow\infty}\varepsilon_n=0$
  and (\ref{theorem-consistency-covariance-estimator-p5}).
  That is, we have proven 
  (\ref{theorem-consistency-covariance-estimator-p1}).\\
  In view of (\ref{theorem-consistency-covariance-estimator-p4}),
  it suffices to prove
  \begin{eqnarray}\label{theorem-consistency-covariance-estimator-p10}
      \frac{1}{n}\!\sum_{i=1}^n 
        g_{D_{n},\lambda_{n},j}(z_i)
        g_{D_{n},\lambda_{n},\ell}(z_i)
      -\frac{1}{n}\!\sum_{i=1}^n 
        g_{P,\lambda_{0},j}(z_i)
        g_{P,\lambda_{0},\ell}(z_i)      
      \;\xrightarrow[\;n\rightarrow\infty\;]{}\;
      0\,.\;\;
  \end{eqnarray}
  in order to prove 
  (\ref{theorem-consistency-covariance-estimator-p2}).
  According to
  (\ref{theorem-consistency-covariance-estimator-p9}),
  (\ref{theorem-consistency-covariance-estimator-p11}), and
  (\ref{theorem-consistency-covariance-estimator-p14}),
  \begin{eqnarray*}
    \lefteqn{
    \bigg|
      \frac{1}{n}\!\sum_{i=1}^n 
        g_{D_{n},\lambda_{n},j}(z_i)
        g_{D_{n},\lambda_{n},\ell}(z_i)
      -\frac{1}{n}\!\sum_{i=1}^n 
        g_{P,\lambda_{0},j}(z_i)
        g_{P,\lambda_{0},\ell}(z_i)      
    \bigg|
    \;\leq}\\
    &\leq&
       \frac{1}{n}\sum_{i=1}^n
       \big|g_{D_{n},\lambda_{n},j}(z_i)-
            g_{P,\lambda_{0},j}(z_i)
       \big|\cdot
       \big|g_{D_{n},\lambda_{n},\ell}(z_i)
       \big|\;+\\
      &&\;\;+\;
       \frac{1}{n}\sum_{i=1}^n
       \big|g_{P,\lambda_{0},j}(z_i)
       \big|\cdot
       \big|g_{D_{n},\lambda_{n},\ell}(z_i)-
            g_{P,\lambda_{0},\ell}(z_i)
       \big|\;\leq\\
    &\leq&\frac{1}{n}\sum_{i=1}^n
            \big(\varepsilon_{n}+\varepsilon_{n}\cdot b_a^\prime(z_i)
            \big)\!\cdot\!
            \big(\varepsilon_{n}+(c+\varepsilon_{n})\cdot b_a^\prime(z_i)
            \big)
            \;+\\
      &&\;\;+\;
       \frac{1}{n}\sum_{i=1}^n
            c\cdot b_a^\prime(z_i)\cdot
            \big(\varepsilon_{n}+\varepsilon_{n}\cdot b_a^\prime(z_i)
            \big)\;=\\
    &=&\varepsilon_n^2+2(\varepsilon_n c+\varepsilon_n^2)\cdot
       \frac{1}{n}\!\sum_{i=1}^n b_a^\prime(z_i)
       +(2\varepsilon_n c+\varepsilon_n^2)\cdot
         \frac{1}{n}\!\sum_{i=1}^n b_a^\prime(z_i)^2
  \end{eqnarray*}
  and the last line converges to 0 as
  $\lim_{n\rightarrow\infty}\varepsilon_n=0$
  and due to (\ref{theorem-consistency-covariance-estimator-p5}).
\end{proof}

\begin{proof}
 \item[\textbf{Proof of Remark 
        \ref{remark-estimator-of-hadamard-derivative}:%
       }
      ]
  According to \cite[Lemma A.9]{hable2012a}, 
  $\G$ is a $P$-Donsker class and, therefore,
  a $P$-Glivenko-Cantelli class almost sure;
  see \cite[p.\ 82]{vandervaartwellner1996}.
  Hence,
  $\mathds{P}_{\mathbf{D}_n}$ converges to $P$ in $B_S$
  almost surely.
  It is shown in \cite[(46) and (47)]{hable2012a} that 
  $S:B_S\rightarrow H,\,\,\mu\mapsto S(\mu)=f_{\mu,\lambda_0}$
  is continuous in $P$ and, therefore,
  $$f_{\mathbf{D}_n,\Lambda_n}
    \;\stackrel{(\ref{app-trace-back-to-standard-reg-parameter})}{=}\;
      S\big(\tfrac{\lambda_0}{\Lambda_n}\mathds{P}_{\mathbf{D}_n}\big)
    \;\xrightarrow[\;n\rightarrow\infty\;]{\text{a.s.}}\;
    S(P)\;=\;f_{P,\lambda_0}\;.
  $$
\end{proof}

\begin{lemma}\label{lemma-non-degeneracy-of-the-limit-marginals}
  Let the Assumptions \ref{basic-assumptions} be fulfilled,
  let $\lambda_{0}\in(0,\infty)$,
  and let $\psi\,:\,\,H\rightarrow\R^m$ be Hadamard-differentiable
  in $\SVMplz$ with derivative $\psi^\prime_{\SVMplz}$.
  For every $j\in\{1,\dots,m\}$, let
  $\psi^\prime_{\SVMplz,j}\in H$ denote the $j$-th component
  of $\psi^\prime_{\SVMplz}$. Let
  $\Sigma_P\in\R^{m\times m}$ be the covariance matrix in
  Corollary \ref{cor-finite-dim-asymptotic-normality}.
  Assume that, for $P_{\X}(dx)$\,-\,a.e.\ $x\in\X$, there
  are $y_1,y_2\in\textup{supp}\big(P(dy|x)\big)$ such that
  \begin{eqnarray}\label{lemma-non-degeneracy-of-the-limit-marginals-1}
    \Ls\big(x,y_1,\SVMplz(x)\big)\,\not=\,
    \Ls\big(x,y_2,\SVMplz(x)\big)\,.
  \end{eqnarray}
  Then, 
  \begin{eqnarray}\label{lemma-non-degeneracy-of-the-limit-marginals-2}
    \Sigma_{P}\text{ has full rank}
    \quad\Leftrightarrow\quad
    \not\exists\, a\in\R^m\setminus\{0\}
    \text{ s.th.\ }
    \,\,a\Transp\psi^\prime_{\SVMplz}=0
    \;\,\,P_{\X}\text{-a.s.}
  \end{eqnarray}
\end{lemma}
\begin{proof}
 \item[\textbf{Proof of Lemma 
       \ref{lemma-non-degeneracy-of-the-limit-marginals}:%
       }
      ]
  According to (\ref{prop-covariance-matrix-2}),
  we have
  \begin{eqnarray}\label{lemma-non-degeneracy-of-the-limit-marginals-p200}
    \Sigma_{P}\text{ has full rank}
    \quad\Leftrightarrow\quad
    \not\exists\,\, a\in\R^m\setminus\{0\},\;
    c\in\R\,:\;\,
    \,\,a\Transp g_{P,\lambda_0}=c
    \;\,\,P\text{-a.s.}
  \end{eqnarray}
  It is a direct consequence of the definition of
  the continuous linear operator $K_P$ that
  $K_P$ is self-adjoint and, accordingly, $K_P^{-1}$
  is again self-adjoint; see \cite[Lemma VI.2.10]{dunford1958}.   
  Hence, according to the reproducing property, we get
  \begin{eqnarray}\label{lemma-non-degeneracy-of-the-limit-marginals-p2}
    \lefteqn{
    g_{P,\lambda_0,j}(x,y)\;=\;
    -\Ls_{\SVMplz}(x,y)\big\la \psi^\prime_{\SVMplz,j},
                               K_P^{-1}\big(\Phi(x)\big)
                       \big\ra
    \;=}\\
    &\!=&\!\!\!
    -\Ls_{\SVMplz}\!(x,y)\big\la K_P^{-1}(\psi^\prime_{\SVMplz,j}),
                               \Phi(x)
                       \big\ra
    \,=\,
    -\Ls_{\SVMplz}\!(x,y)\big[\!K_P^{-1}(\psi^\prime_{\SVMplz,j})\!
                       \big]\!(x) \nonumber
  \end{eqnarray}
  and, therefore,
  \begin{eqnarray}\label{lemma-non-degeneracy-of-the-limit-marginals-p201}
    a\Transp g_{P,\lambda_0}\;=\;
    -\Ls_{\SVMplz}K_P^{-1}(a\Transp\psi^\prime_{\SVMplz})
    \qquad\forall\,a\in\R^m\;.
  \end{eqnarray}
  It will be shown below that, for every $f\in H$,
  \begin{eqnarray}\label{lemma-non-degeneracy-of-the-limit-marginals-p202}
    f=0\;\,\,P\text{-a.s.}
    \qquad\Leftrightarrow\qquad
    K_P^{-1}(f)=0\;\,\,P\text{-a.s.}
  \end{eqnarray}
  By use of these preparations and 
  (\ref{lemma-non-degeneracy-of-the-limit-marginals-p202}), the proof
  of (\ref{lemma-non-degeneracy-of-the-limit-marginals-2}) can be done
  quickly: First, assume that there is an
  $a\in\R^m\setminus\{0\}$ such that
  $a\Transp\psi^\prime_{\SVMplz}=0\;\,\,P_{\X}\text{-a.s.}$
  Then, it follows from 
  (\ref{lemma-non-degeneracy-of-the-limit-marginals-p202}),
  (\ref{lemma-non-degeneracy-of-the-limit-marginals-p201}), and
  (\ref{lemma-non-degeneracy-of-the-limit-marginals-p200})
  that $\Sigma_P$ does not have full rank. That is we have proven
  ``$\Rightarrow$'' in 
  (\ref{lemma-non-degeneracy-of-the-limit-marginals-2}).
  Next, in order to prove ``$\Leftarrow$'' in 
  (\ref{lemma-non-degeneracy-of-the-limit-marginals-2}),
  assume that $\Sigma_P$ does not have full rank. Then,
  according to (\ref{lemma-non-degeneracy-of-the-limit-marginals-p200})
  and (\ref{lemma-non-degeneracy-of-the-limit-marginals-p201}),
  there is an $a\in\R^m\setminus\{0\}$ and a $c\in\R$ such that,
  for $P_{\X}(dx)$-a.e.\ $x\in\X$  
  $$-\Ls_{\SVMplz}(x,\cdot)
     \big[K_P^{-1}(a\Transp\psi^\prime_{\SVMplz})\big](x)
    \,=\,c
    \;\;\;\;P(\cdot|x)\text{-a.s.}
  $$
  Hence, for $P_{\X}(dx)$-a.e.\ $x\in\X$,
  it follows from 
  (\ref{lemma-non-degeneracy-of-the-limit-marginals-1})
  and continuity of $y\mapsto \Ls_{\SVMplz}(x,y)$ 
  that 
  \begin{eqnarray}\label{lemma-non-degeneracy-of-the-limit-marginals-p3}
    \big[K_P^{-1}(a\Transp\psi^\prime_{\SVMplz})\big](x)
    \;=\;0\;.
  \end{eqnarray}
  According to (\ref{lemma-non-degeneracy-of-the-limit-marginals-p202}),
  this implies that $a\Transp\psi^\prime_{\SVMplz}=0$
  $P_{\X}$-a.s. That is, we have proven
  ``$\Leftarrow$'' in 
  (\ref{lemma-non-degeneracy-of-the-limit-marginals-2}).
  
  Now, it only remains to prove statement 
  (\ref{lemma-non-degeneracy-of-the-limit-marginals-p202}).
  To this end, define
  $\overline{\X}:=\text{supp}(P_{\X})$, 
  let $P_{\overline{\X}}$ be the restriction
  of $P_{\X}$ on the Borel-$\sigma$-algebra of $\overline{\X}$, and
  let $\overline{P}$ be the probability measure on
  $\overline{\X}\times \Y$ defined by
  $$\overline{P}(B)
    =\int\int I_B(x,y)\,P(dy|x)\,P_{\overline{\X}}(dx)
  $$
  for every $B$ in the Borel-$\sigma$-algebra of
  $\overline{\X}\times\Y$.
  In addition, let $\overline{k}$ be the restriction of the kernel $k$
  on $\overline{\X}\times\overline{\X}$
  and $\overline{\Phi}$ the corresponding canonical feature
  map. Then, the
  RKHS of $\overline{k}$ is
  $$\overline{H}\;:=\;
    \big\{\overline{f}:\overline{\X}\rightarrow\R\,
    \big|\,\overline{f}\text{ is the restriction of an }f\in H
           \text{ on }\overline{\X}
    \big\}\,;
  $$
  see e.g.\ \cite[\S\,4.2]{berlinet2004}. For every $f\in H$, let
  $\overline{f}$ denote the restriction of $f$ on
  $\overline{\X}$. 
  Define
  $$\overline{K_P}\,:\;\;\overline{H}\;\rightarrow\;\overline{H}\,,
    \qquad \overline{f}\;\mapsto\;
    2\lambda_0 \overline{f}
    +\int \Ls_{\SVMplz}(x,y)\overline{f}(x)\overline{\Phi}(x)
     \,\overline{P}\big(d(x,y)\big)\;.  
  $$
  As Assumption \ref{basic-assumptions}
  is also fulfilled for $\overline{\X}$ and $\overline{P}$ 
  instead of $\X$ and $P$, it follows from 
  \cite[Lemma A.5]{hable2012a} that
  $\overline{K_P}$ is invertible. 
  The definitions imply 
  $\overline{K(f)}=\overline{K_P}\left(\overline{f}\right)$
  for every $f\in H$ and, therefore,
  $$\overline{K_P}\Big(\overline{K_P^{-1}(f)}\Big)
    \;=\;\overline{K_P\big(K_P^{-1}(f)\big)}\;=\;\overline{f}
    \qquad\forall\,f\in H\;.
  $$
  Hence
  \begin{eqnarray}\label{lemma-non-degeneracy-of-the-limit-marginals-p5}
    \overline{K_P^{-1}(f)}\;=\;
    {\overline{K_P}}^{-1}\big(\overline{f}\big)
    \qquad\forall\,\,\overline{f}\in\overline{H}\;.
  \end{eqnarray}
  Since $\overline{k}$ is continuous and
  $\text{supp}(P_{\overline{\X}})=\overline{\X}$, it follows from
  \cite[Exercise 4.6]{steinwart2008} 
  that, for every $f\in H$,
  \begin{eqnarray}\label{lemma-non-degeneracy-of-the-limit-marginals-p6}
    f\,=\,0\quad P_{\X}\text{-a.s.}
    \qquad\Leftrightarrow\qquad
    \overline{f}\,=\,0
  \end{eqnarray}
  Hence, for every $f\in H$,
  \begin{eqnarray*}
    K_P^{-1}(f)
    \,=\,0\;\;\;P_{\X}\text{-a.s.}
    &
    \stackrel{(\ref{lemma-non-degeneracy-of-the-limit-marginals-p6})}{
              \Leftrightarrow
             }
    &
    \overline{K_P^{-1}(f)}\,=\,0
    \quad
      \stackrel{(\ref{lemma-non-degeneracy-of-the-limit-marginals-p5})}{
                \Leftrightarrow
               }
     \quad
       \overline{K_P}^{-1}(\overline{f})\,=\,0
       \quad\Leftrightarrow
       \\ 
     &\Leftrightarrow &
       \overline{f}\,=\,0
       \quad
       \stackrel{(\ref{lemma-non-degeneracy-of-the-limit-marginals-p6})}{
                \Leftrightarrow
               }
       \quad
       f
       \,=\,0\;\;\;P_{\X}\text{-a.s.}
  \end{eqnarray*} 
\end{proof}

\begin{proof}
 \item[\textbf{Proof of Theorem 
       \ref{theorem-asymptotic-confidence-sets}:%
       }
      ]
  Since taking the square root of a 
  symmetric positive definite matrix is continuous,
  see e.g.\
  \cite[\S\,7.8, Exercise 1]{serre2002},
  it follows from
  Theorem \ref{theorem-consistency-covariance-estimator},
  that
  $\hat{\Sigma}_n(\mathbf{D}_n,\Lambda_n)^{\frac{1}{2}}
   \longrightarrow\Sigma_P^{\frac{1}{2}}
  $
  almost surely for $n\rightarrow\infty$. Hence, 
  Corollary \ref{cor-finite-dim-asymptotic-normality} yields
  $$\sqrt{n}\cdot\hat{\Sigma}_n(\mathbf{D}_n,\Lambda_n)^{-\frac{1}{2}}
    \Big(\psi\big(f_{\mathbf{D}_{n},\Lambda_{n}}
             \big)-
         \psi\big(\SVMplz\big)
    \Big)
    \;\;\leadsto\;\;
    \mathcal{N}_m\big(0,\textup{Id}_{m\times m}\big)\,;
  $$
  see e.g.\ \cite[p.\ 11]{vandervaart1998}. Finally,
  weak convergence, the continuous mapping theorem,
  the portmanteau theorem, and the definition of the 
  chi-squared distribution imply
  \begin{eqnarray*}
    \lefteqn{
    \lim_{n\rightarrow\infty}Q\Big(\psi\big(\SVMplz\big)
           \,\in\,
          C_{n,\alpha}(\mathbf{D}_n,\Lambda_n)
     \Big)
    \;=
    }\\
    &=&\lim_{n\rightarrow\infty}
       Q\Big(\big\|\sqrt{n}\cdot
               \hat{\Sigma}_n(\mathbf{D}_n,\Lambda_n)^{-\frac{1}{2}}
               \big(\psi(\SVMplz)-\psi(f_{\mathbf{D}_{n},\Lambda_{n}})
               \big)
             \big\|^2_{\R^m}
               \leq \, \chi_{m,\alpha}^2
         \Big)\;=\\
     &=&1-\alpha\;.
  \end{eqnarray*}
\end{proof}

\begin{proof}
 \item[\textbf{Proof of Prop.\ 
       \ref{prop-calculation-of-covariance-estimator}:%
       }
      ]
  Let $\{\Phi(x_{i_1}),\dots,\Phi(x_{i_r})\}$ be the maximal
  linearly independent subset of 
  $\{\Phi(x_{1}),\dots,\Phi(x_{n})\}$ which
  defines $B_{D_n}$ according to
  (\ref{prop-calculation-of-covariance-estimator-prep-1})
  and
  (\ref{prop-calculation-of-covariance-estimator-prep-2}).
  Fix  any
  $x\in\X$ and any $y\in\Y$. 
  We have to find an 
  $f\in H$ such that
  $K_{D_n,\lambda}(f)=\Phi(x)$. 
  (The solution $f$
  depends on $D_n$ and $\lambda$
  though this is not
  made explicit in the notation.) Hence,
  by using $f(x_i)=\la f,\Phi(x_i)\ra$,
  \begin{eqnarray}\label{prop-calculation-of-covariance-estimator-p1}
    \Phi(x)\;=\;K_{D_n,\lambda}(f)\;=\;
    2\lambda f+
    \frac{1}{n}\sum_{i=1}^n\Lss_{f_{D_n,\lambda}}(x_i,y_i)
                           \la f,\Phi(x_i)\ra\Phi(x_i)\;.
  \end{eqnarray}
  Rearranging this equality yields
  $$f\;=\;
    \frac{1}{2\lambda}\Phi(x)-
    \frac{1}{2n\lambda}\sum_{i=1}^n
                           \Lss_{f_{D_n,\lambda}}(x_i,y_i)
                           \la f,\Phi(x_i)\ra\Phi(x_i)
  $$
  and, therefore,
  \begin{eqnarray}\label{prop-calculation-of-covariance-estimator-p2}
    f\,=\,\frac{1}{2\lambda}\Phi(x)+h
    \qquad\text{for some }\,h\in
    \text{lin}\big\{\Phi(x_1),\dots,\Phi(x_n)\big\}\,.
  \end{eqnarray}
  Define
  $$w_i\;:=\;
    -\frac{1}{2n\lambda}
              \Lss_{f_{D_n,\lambda}}(x_i,y_i)
              k(x_i,x)
    \qquad\forall\,i\in\{1,\dots,n\}
  $$
  and $w:=(w_1,\dots,w_n)^{\scriptscriptstyle\mathsf{T}}$.
  Putting (\ref{prop-calculation-of-covariance-estimator-p2})
  into (\ref{prop-calculation-of-covariance-estimator-p1}) again
  and a simple rearranging of the resulting equation 
  lead to
  \begin{eqnarray}\label{prop-calculation-of-covariance-estimator-p3}
    2\lambda h
    +\frac{1}{n}\sum_{i=1}^n
        \Lss_{f_{D_n,\lambda}}(x_i,y_i)
           \la h,\Phi(x_i)\ra\Phi(x_i)
    \;=\;\sum_{i=1}^n w_i\Phi(x_i)\;.
  \end{eqnarray}
  That is, $f$ solves (\ref{prop-calculation-of-covariance-estimator-p1})
  if and only if $f$ is of form 
  (\ref{prop-calculation-of-covariance-estimator-p2})
  where $h$ solves
  (\ref{prop-calculation-of-covariance-estimator-p3}). 
  Next, define the linear map
  $$\gamma\;:\;\;
    \text{lin}\big\{\Phi(x_1),\dots,\Phi(x_n)\big\}
    \;\rightarrow\;
    \text{lin}\big\{\Phi(x_1),\dots,\Phi(x_n)\big\}
  $$
  by
  $$\gamma(h)\;=\;2\lambda h
    +\frac{1}{n}\sum_{i=1}^n
        \Lss_{f_{D_n,\lambda}}(x_i,y_i)
           \la h,\Phi(x_i)\ra\Phi(x_i)    
  $$
  for every
  $h\in\text{lin}\big\{\Phi(x_1),\dots,\Phi(x_n)\big\}$. 
  That is, in order to find $h$ which fulfills 
  (\ref{prop-calculation-of-covariance-estimator-p3}) we have to
  find $\alpha_1,\dots,\alpha_n\in\R$
  such that
  \begin{eqnarray}\label{prop-calculation-of-covariance-estimator-p4}
    \gamma\bigg(\sum_{i=1}^n\alpha_i\Phi(x_i)\bigg)
    \;=\;\sum_{i=1}^n w_i\Phi(x_i)\;.
  \end{eqnarray}
  Existence of a solution $h$ and therefore, of 
  $\alpha_1,\dots,\alpha_n$
  is guaranteed as $K_{D_n,\lambda}$ is invertible.
  Let $a_{\ell i}$ be the $(\ell,i)$-entry of the
  matrix $A_{D_n,\lambda}$, $\ell,i\in\{1,\dots,n\}$.
  According to the definition of $A_{D_n,\lambda}$,
  \begin{eqnarray}\label{prop-calculation-of-covariance-estimator-p6}
    \gamma\big(\Phi(x_i)\big)\;=\;\sum_{\ell=1}^n a_{\ell i}\Phi(x_\ell)
    \qquad\forall\,i\in\{1,\dots,n\}\;.
  \end{eqnarray}
  It follows from 
  $$\sum_{i=1}^n w_i\Phi(x_i)
    \;\stackrel{(\ref{prop-calculation-of-covariance-estimator-prep-1})
               }{=}\;
    \sum_{i=1}^n w_i\!\cdot\!
                 \bigg(\sum_{j=1}^r\beta_{ji}\Phi(x_{i_j})\bigg)
    \;=\;
    \sum_{j=1}^r\bigg(\sum_{i=1}^n\beta_{ji}w_i\bigg)\Phi(x_{i_j})
  $$
  and 
  \begin{eqnarray*}
    \lefteqn{
    \gamma\bigg(\sum_{i=1}^n\alpha_i\Phi(x_i)\bigg)\;=\;
       \sum_{i=1}^n\alpha_i\gamma\big(\Phi(x_i)\big)
       \;\stackrel{(\ref{prop-calculation-of-covariance-estimator-p6})
                  }{=}\; 
       \sum_{i=1}^n\alpha_i\sum_{\ell=1}^n a_{\ell i}\Phi(x_\ell)\;=
    } \\
    &&\stackrel{(\ref{prop-calculation-of-covariance-estimator-prep-1})
               }{=}\;
       \sum_{i=1}^n\alpha_i\sum_{\ell=1}^n a_{\ell i}
          \sum_{j=1}^r\beta_{j\ell}\Phi(x_{i_j})\;=\;
       \sum_{j=1}^r
          \bigg(\sum_{i=1}^n\sum_{\ell=1}^n 
                \beta_{j\ell}a_{\ell i}\alpha_i
          \bigg)
          \Phi(x_{i_j})
  \end{eqnarray*}
  that 
  $\alpha:=(\alpha_1,\dots,\alpha_n)^{\scriptscriptstyle\mathsf{T}}$
  is a solution of (\ref{prop-calculation-of-covariance-estimator-p4})
  if and only if
  \begin{eqnarray}\label{prop-calculation-of-covariance-estimator-p7}
    \sum_{j=1}^r
          \bigg(\sum_{i=1}^n\sum_{\ell=1}^n 
                \beta_{j\ell}a_{\ell i}\alpha_i
          \bigg)
          \Phi(x_{i_j})
    \;=\;
    \sum_{j=1}^r\bigg(\sum_{i=1}^n\beta_{ji}w_i\bigg)\Phi(x_{i_j})\;.
  \end{eqnarray}
  Linear independence 
  of $\Phi(x_{i_1}),\dots,\Phi(x_{i_r})$ 
  implies that 
  (\ref{prop-calculation-of-covariance-estimator-p7}) is
  equivalent to
  $$\sum_{i=1}^n\beta_{ji}w_i
    \;=\;
    \sum_{i=1}^n\sum_{\ell=1}^n 
                \beta_{j\ell}a_{\ell i}\alpha_i
    \qquad\forall\,j\in\{1,\dots,r\}
  $$
  or, in matrix notation,
  \begin{eqnarray}\label{prop-calculation-of-covariance-estimator-p8}
    B_{D_n}\cdot w
    \;=\;B_{D_n}A_{D_n,\lambda}\cdot\alpha\;.
  \end{eqnarray}
  Summing up, we have proven that
  $\alpha\in\R^n$ is a solution of 
  (\ref{prop-calculation-of-covariance-estimator-p4})
  if and only if $\alpha$ solves 
  (\ref{prop-calculation-of-covariance-estimator-p8}).
  As already stated above, a solution of 
  (\ref{prop-calculation-of-covariance-estimator-p4})
  and, therefore, of 
  (\ref{prop-calculation-of-covariance-estimator-p8})
  exists. Hence,
  $$\alpha\;:=\;(B_{D_n}A_{D_n,\lambda})^- B_{D_n}w
  $$
  solves (\ref{prop-calculation-of-covariance-estimator-p8})
  and, therefore
  (\ref{prop-calculation-of-covariance-estimator-p4}).
\end{proof}

\section*{Acknowledgment}

I would like to thank Andreas Christmann and Thoralf
Mildenberger for discussions and valuable suggestions.

\bibliographystyle{abbrvnat}
\bibliography{literatur}

\end{document}